\documentclass[twoside]{article}
\usepackage[utf8]{inputenc}
\usepackage[citestyle=numeric-comp]{biblatex}
\bibliography{main}
\usepackage{amsmath, amsfonts, amsthm, amssymb}
\usepackage{graphicx}
\usepackage{caption,subcaption}
\usepackage[margin=1in]{geometry}

\newtheorem{theorem}{Theorem}
\newtheorem{lemma}{Lemma}
\newtheorem{problem}{Problem}

\usepackage[affil-it]{authblk}

\title{Online estimation and control with optimal pathlength regret}

\author{Gautam Goel%
  \thanks{Email: ggoel@caltech.edu}}
\affil{Department of Computing and Mathematical Sciences \\ Caltech}

\author{Babak Hassibi%
  \thanks{Email: hassibi@caltech.edu}}
\affil{Department of Electrical Engineering \\Caltech}

\date{}
\begin{document}

\maketitle

\begin{abstract}
    A natural goal when designing online learning algorithms for non-stationary environments is to bound the regret of the algorithm in terms of the temporal variation of the input sequence. Intuitively, when the variation is small, it should be easier for the algorithm to achieve low regret, since past observations are predictive of future inputs. Such data-dependent ``pathlength" regret bounds have recently been obtained for a wide variety of online learning problems, including online convex optimization (OCO) and bandits. We obtain the first pathlength regret bounds for online control and estimation (e.g. Kalman filtering) in linear dynamical systems. The key idea in our derivation is to reduce pathlength-optimal filtering and control to certain variational problems in robust estimation and control; these reductions may be of independent interest.  Numerical simulations confirm that our pathlength-optimal algorithms outperform traditional $H_2$ and $H_{\infty}$ algorithms when the environment varies over time.
\end{abstract}

\section{Introduction}
Online learning has traditionally focused on minimizing regret against static policies. For example, in the multi-armed bandit problem, we compare the rewards obtained by the online learner to the rewards that the learner could have counterfactually obtained if they had simply pulled the same arm in each round. However, many real-world environments are non-stationary, e.g. the rewards associated to actions vary over time. In such settings, it is more natural to compare the actions of the online learner against a time-varying sequence of actions, rather than a single fixed action. It is also natural to bound the regret of the online algorithm in terms of the temporal variation in the reward sequence. Intuitively, when the variation is small, it should be easier for the algorithm to achieve low regret, since past observations are predictive of future rewards. Such data-dependent ``pathlength" regret bounds have recently been obtained for a wide variety of online learning problems, including online convex optimization (OCO) and bandits, e.g. \cite{zinkevich2003online, wei2018more, bubeck2019improved, goel2019online}.

A particularly challenging problem not considered in these prior works is obtaining pathlength regret bounds for online algorithms in environments with underlying dynamics. Intuitively, learning is harder in such settings because the rewards and actions are tightly coupled across rounds via the state; selecting a poor action in one round affects the rewards obtained in all subsequent rounds. This is in stark contrast to classical online learning problems (e.g. bandits), where a single poor decision results in a low reward for the algorithm in that round, but does not directly affect the rewards in future rounds. We ask: is it possible to design
algorithms for online decision-making in dynamical systems which have regret bounded by pathlength? Answering this question has taken on a newfound urgency with the recent deployment of autonomous control systems such as drones and self-driving cars; these systems invariably encounter non-stationarity in their environments (e.g. changing weather, shifting traffic patterns, etc.)

\subsection{Contributions of this paper}
In this paper, we obtain the first pathlength regret bounds for online control and estimation (e.g. Kalman filtering) in linear dynamical systems. Our results show that it possible to design controllers and filters which dynamically adapt to regularity in the disturbance sequence, even though the rewards are strongly coupled across time by dynamics. We construct a controller whose regret against a clairvoyant offline optimal controller has optimal dependence on the pathlength of the driving disturbance. Similarly, we construct a filter whose regret against the optimal smoothed estimator has optimal dependence on the pathlength of the measurement disturbance. The key idea underpinning our results is to reduce pathlength-optimal filtering and control to certain variational problems widely studied in the context of robust estimation and control during the 1980s. In Section \ref{control-sec}, we show that the problem of obtaining a pathlength-optimal controller in the original dynamical system can be reduced to the problem of obtaining the $H_{\infty}$-optimal controller in a specially constructed synthetic system. Similarly, in Section \ref{filtering-sec} we reduce pathlength-optimal filtering to the classical Nehari problem from robust control. This problem, first introduced in 1957 in an operator theoretic context by Zeev Nehari \cite{nehari1957bounded}, asks how closely a noncausal function can be approximated by a causal one. We describe the pathlength-optimal filter in terms of a computationally efficient state-space solution to the Nehari problem. In Section \ref{experiments-sec}, we present numerical experiments which confirm that our pathlength-optimal algorithms outperform traditional $H_2$ and $H_{\infty}$ algorithms when the environment varies over time. As a bonus, the algorithms we obtain also have regret with optimal dependence on the energy of the disturbances, up to a factor of 4. Our results show that classical techniques from $H_{\infty}$ estimation and control, which were originally designed with the aim of ensuring \textit{robustness}, can instead be used to obtain \textit{adaptivity}; we believe such techniques deserve to be better known in learning and control.

%


\subsection{Related work}

The idea of bounding dynamic regret by pathlength was introduced in the seminal work of Zinkevich \cite{zinkevich2003online}; this idea was further explored in \cite{chiang2012online}. Dynamic regret bounds in terms of pathlength were obtained for the multi-armed bandit problem in \cite{besbes2015non, wei2016tracking}. More recently, static regret bounds (i.e. regret against the best fixed arm) in terms of pathlength were obtained in \cite{wei2018more, bubeck2019improved}. Pathlength regret bounds for online convex optimization with switching costs were obtained in \cite{goel2019online}.

The problem of designing controllers with optimal dynamic regret was studied in the finite-horizon, time-varying setting in \cite{goel2021regret}, in the infinite-horizon LTI setting in \cite{sabag2021regret}, and in the measurement-feedback setting in \cite{goel2021measurement}. These works all bounded regret by the energy in the disturbances; the pathlength regret bounds we obtain in this paper also imply energy regret bounds which are optimal up to a factor of 4. Filtering algorithms with energy regret bounds were obtained in the finite-horizon setting in \cite{goel2021regret} and the infinite-horizon setting in \cite{sabag2021filtering}. Gradient-based control algorithms with low dynamic regret against the class of disturbance-action policies were obtained in \cite{  zhao2021non}; the stronger metric of adaptive regret was studied in \cite{gradu2020adaptive}. We also note recent work \cite{goel2019online, goel2021competitive}, which considered control through the lens of competitive ratio, a multiplicative analog of dynamic regret.

\section{Preliminaries}

\subsection{Signals, energy and pathlength}
The energy in a signal $w = (\ldots, w_{-1}, w_0, w_1, w_2, \ldots)$ is the squared $\ell_2$ norm of $w$: $$\textsc{energy}(w) = \|w\|_2^2 = \sum_{t=-\infty}^{\infty} \|w_t\|_2^2. $$
The pathlength of $w$ is $$\textsc{pathlength}(w) = \sum_{t=-\infty}^{\infty} \|w_t - w_{t-1}\|_2^2.$$ Intuitively, the energy of $w$ measures how ``large" $w$ is, while the pathlength measures the temporal variability in $w$. In the $z$-domain, the pathlength is $\|M(z)w(z)\|_2^2$, where $M(z) = 1-z$. A simple application of the Cauchy-Schwarz and AM-GM inequalities show that for all $w$, 
\begin{equation} \label{pathlength-energy-inequality}
\textsc{pathlength}(w) \leq 4 \cdot \textsc{energy}(w).
\end{equation}
In this sense, pathlength is a stronger metric than energy; given a pathlength regret bound, one immediately obtains a corresponding energy regret bound, but it is not possible, in general, to go in the opposite direction. 

We say that a function is a \textit{causal} function of a signal $w$ if its output at time $t$ depends only on $\{w_s\}_{s \leq t}$, and it is anticausal if it depends only on $\{w_s\}_{s \geq t}$. A function is noncausal if it is not causal, i.e. its output at time $t$ depends on some $w_{t + k}$.

\subsection{Control setting}
We focus on the linear-quadratic (LQ) control setting, where a linear time-invariant (LTI) system evolves according to the dynamics $$x_{t+1} = A x_t + B_u u_t + B_w w_t,$$ where $x_t \in \mathbb{R}^n$ is the state, $u_t \in \mathbb{R}^m$ is the control and $w_t \in \mathbb{R}^p$ is an external disturbance. 
We incur a quadratic cost $$x_t^* Q x_t + u_t^* R u_t $$ in each round, where $Q \succeq 0, R \succ 0$. Naturally, our goal is to select the control actions $u$ so as to minimize the aggregate cost across rounds (note that cost minimization can be reframed as reward maximization). As is standard in infinite-horizon control, we assume that $(A, B_u)$ is stabilizable and $(A, Q^{1/2})$ is detectable. 
We consider both the causal setting, where the control $u_t$ is selected after observing the state $x_t$ and the disturbance $w_t$, and the strictly causal setting, where the control $u_t$ is selected after observing  $x_t$ but before observing $w_t$.

It is convenient to reparameterize the system dynamics and costs as follows. Let $L$ be a square-root of $Q$ (e.g. $L^*L = Q$) and let $R^{1/2}$ denote a square-root of $R$. Define $s_t = Lx_t,$ $v_t = R^{1/2}u_t$; note that we can easily recover $u_t$ from $v_t$ by setting $u_t = R^{-1/2}v_t$. With this notation, the cost incurred by an online algorithm which selects $v$ in response to $w$ is $$ALG(w) = \|s\|_2^2 + \|v\|_2^2.$$ The dynamics can also be rewritten in terms of $s_t$ and $v_t$: $$x_{t+1} = A x_t + B_uR^{-1/2}v_t + B_w w_t, \hspace{3mm} s_t = Lx_t.$$ We can cleanly capture the dynamics in terms of transfer matrices as $$ s = Fv + Gw$$ where $F, G$ are strictly causal transfer matrices encoding $\{A, B_uR^{-1/2}, B_w, L \}$. 

Our goal is to design a controller which minimizes regret against a clairvoyant offline optimal controller $K_0$ which selects the optimal control given perfect noncausal knowledge of $w$. One can show (Theorem 11.2.1 in \cite{hassibi1999indefinite}) that the offline optimal control policy is
\begin{equation}  \label{offline-operator}
K_0(w) = - (I + F^*F)^{-1}F^*Gw,
\end{equation}
and the corresponding offline optimal cost is 
\begin{equation} \label{offline-operator-cost}
OPT(w) = w^*G^*(I + FF^*)^{-1}Gw.
\end{equation}

Formally, we seek to solve the following problem:
\begin{problem} [Pathlength-optimal control at level $\gamma$] \label{pathlength-optimal-control-problem}
Given $\gamma > 0$, find a causal (or strictly causal) controller $K$  whose cost, $ALG(w)$, satisfies \begin{equation} \label{control-regret-gamma-cond}
ALG(w) - OPT(w) < \gamma^2 \cdot \textsc{pathlength}(w),
\end{equation} 
or determine whether no such controller exists.
\end{problem}

Once this feasibility problem  is solved, it is easy to recover the optimal value of $\gamma$ via bisection; we define  the \textit{pathlength-optimal controller} to be the controller which satisfies (\ref{control-regret-gamma-cond}) with the smallest possible $\gamma$.

Our proof technique is to reduce the problem of finding a controller satisfying (\ref{control-regret-gamma-cond}) to an $H_{\infty}$ control problem:

\begin{problem} [$H_{\infty}$-optimal control at level $\gamma$]
Given $\gamma > 0$, find a causal (or strictly causal) controller $K$  whose cost, $ALG(w)$, satisfies $$ALG(w) < \gamma^2 \cdot \textsc{energy}(w), $$ or determine whether no such controller exists.
\end{problem}

This problem has a well-known solution:

\begin{theorem}[Theorem 13.3.3 in \cite{hassibi1999indefinite}] \label{hinf-ih-controller-thm}
Suppose $(A, B_u)$ is stabilizable and $(A, Q^{1/2})$ is observable on the unit circle. A causal controller at level $\gamma$ exists if and only if there exists a solution to the Ricatti equation
\begin{equation*}
P = Q + A^*PA -A^* P \tilde{B}\tilde{H}^{-1}\tilde{B}^*PA
\end{equation*}
with $$ \tilde{B} = \begin{bmatrix} B_{u} & B_{w} \end{bmatrix}, \hspace{3mm} \tilde{R} = \begin{bmatrix} I & 0 \\ 0 & -\gamma^2 I \end{bmatrix},$$ $$\tilde{H} =  \tilde{R}  + \tilde{B}^* P \tilde{B}, $$
such that 
\begin{enumerate}
    \item $ A - \tilde{B}\tilde{H}^{-1}\tilde{B}^*PA$ is stable;
    \item $\tilde{R}$ and $\tilde{H}$ have the same inertia;
    \item $P \succeq 0$.
\end{enumerate}
In this case, the infinite-horizon $H_{\infty}$ controller at level $\gamma$ has the form
\begin{equation*}
u_t =  -H^{-1} B_{u}^*P(Ax_t +  B_{w}  w_t),
\end{equation*}
where $H = I + B_u^*PB_u$. A strictly causal $H_{\infty}$ controller at level $\gamma$ exists if and only if conditions 1 and 3 hold, and additionally $$B_u^*PB_u \prec \gamma^2 I, \hspace{3mm} I + B_w^*P(I - B_u(-\gamma^2 I + B_u^*PB_u)^{-1}B_u^*P)B_w \succ 0.$$ In this case, one possible strictly causal $H_{\infty}$ controller at level $\gamma$ is given by \begin{equation*}
u_t =  -H^{-1} B_{u}^*PAx_t.
\end{equation*}
\end{theorem}

\subsection{Filtering setting}
We consider a linear time-invariant (LTI) system that evolves according to the dynamics $$x_{t+1} = A x_t + B w_t , \hspace{3mm} y_t = C x_t + v_t$$ where $x_t \in \mathbb{R}^n$ is the state, $y_t$ is a noisy linear measurement of the state, and $w_t \in \mathbb{R}^m$ and $v_t \in \mathbb{R}^p$ are external disturbances. Our goal is to estimate a linear function of the states: $$ s_t = Lx_t,$$
given the measurements, where $L \in \mathbb{R}^r$. The special case where $r = n$ and $L = I$ corresponds to estimating the state itself. In each round, we output an estimate $\hat{s}_t$ based on the current observation $y_t$ and all previous observations and incur a squared-error loss $$\|\hat{s}_t - s_t\|_2^2. $$  Naturally, our goal is to select the estimates so as to minimize the aggregate loss across all rounds. As is standard in infinite-horizon estimation, we assume that $(A, B)$ is stabilizable and $(A, C)$ is detectable. 

In this paper, we obtain the pathlength-optimal filter using frequency domain analysis. We let $J(z)$ be the transfer matrix mapping the disturbance $w$ to the signal $s$: $$ J(z) = L(zI - A)^{-1}B$$ and let $H(z)$ be the transfer matrix mapping the disturbance $w$ to the observations $y$: $$H(z) = C(zI - A)^{-1}B.$$ Any filter $K$ naturally induces a transfer matrix $T_K$ which maps the disturbances $w$ and $v$ to the estimation error $\hat{s} - s$ incurred by $K$: 
$$T_K(z) =  \begin{bmatrix} J(z) - K(z)H(z) & -K(z) \end{bmatrix}. $$

Our goal is to design a filter $K$ which minimizes regret against the optimal smoothed estimator, i.e. the estimator which minimizes estimation error given noncausal access to the observations $y$. One can show (Theorem 10.3.1 in \cite{hassibi1999indefinite}) that the optimal smoothed estimator in the $z$-domain is
\begin{equation}  \label{optimal-smoothed-estimate}
K_0(z) = J(z)H^*(z^{-*})(I + H(z)H^*(z^{-*}))^{-1}.
\end{equation}

In this paper, we focus on bounding the regret by the pathlength of the measurement disturbance $v$ and the energy of $w$; we leave the problem of bounding regret by the pathlength of both $w$ and $v$ for future work. Formally, we seek to solve the following problem:
\begin{problem} [Pathlength-optimal filtering at level $\gamma$] \label{pathlength-estimation-problem}
Given $\gamma > 0$, find a filter $K$  such that the regret \begin{equation*} 
\begin{bmatrix} w^* & v^* \end{bmatrix} (T_K^*T_K - T_{K_0}^*T_{K_0}) \begin{bmatrix} w \\ v \end{bmatrix} 
\end{equation*} 
is at most 
\begin{equation} \label{filter-regret-gamma-cond}
\gamma^2 \cdot (\textsc{energy}(w) + \textsc{pathlength}(v)), 
\end{equation}
or determine whether no such filter exists.
\end{problem}

Once this feasibility problem  is solved, it is easy to recover the optimal value of $\gamma$ via bisection; we define  the \textit{pathlength-optimal filter} to be the filter which satisfies (\ref{filter-regret-gamma-cond}) with the smallest possible value of $\gamma$.

Our proof technique reduces the problem of finding a filter satisfying (\ref{filter-regret-gamma-cond}) to Nehari problem:

\begin{problem} [Nehari problem at level $\gamma$]
Given $\gamma > 0$ and a strictly anticausal function $T(z)$, find a causal and bounded function $K(z)s$ such that \begin{equation} \label{nehari-gamma-cond}
    \|K(z) - T(z)\|_{\infty} < \gamma,
\end{equation}
or determine whether no such $K(z)$ exists.
\end{problem}

Recall that $$ \|K(z) - T(z)\|_{\infty} = \sup_{\theta \in [0, 2\pi]} \bar{\sigma} \left(K(e^{i \theta}) - T(e^{i \theta})\right) .$$ Informally, the Nehari problem seeks to find a causal function $K(z)$ which is $\gamma$-close to the anticausal function $T(z)$ at every frequency $\theta \in [0, 2\pi]$. 

When the anticausal function $T(z)$ has the rational form $T(z) = H(zI - F)^{-1}G$, where $F$ is stable, the Nehari problem has an explicit solution:

\begin{theorem}[Theorem 9 in \cite{sabag2021regret}] \label{nehari-solution-theorem}
The smallest value of $\gamma$ such that there exists a causal and bounded $K(z)$ satisfying (\ref{nehari-gamma-cond}) is $\gamma^{\star} = \bar{\sigma}(Z\Pi)$, where $Z$ and $\Pi$ are solutions of the Lyapunov equations $$Z = F^*ZF + H^*H, \hspace{3mm} \Pi = F \Pi F^* + GG^*. $$ For all $\gamma \geq \gamma^{\star}$, a corresponding $\gamma$-optimal solution $\hat{K}(z)$ to (\ref{nehari-gamma-cond})  is $$\hat{K}(z) = H\Pi(I + F_{\gamma}(zI - F_{\gamma})^{-1})K_{\gamma}, $$
where $$F_{\gamma} = F^* - K_{\gamma}G^*, \hspace{3mm} K_{\gamma} = (I - F^* Z_{\gamma} F \Pi)^{-1} F^* Z_{\gamma} G,  $$ and $Z_{\gamma}$ is the solution of the Lyapunov equation $$Z_{\gamma} = F^*Z_{\gamma}F + \gamma^{-2} H^*H. $$
\end{theorem}

\subsection{Notation}
We let $\bar{\sigma}(A)$ denote the largest singular value of $A$. To improve legibility, we often drop the explicit dependence on $z$ when describing transfer matrices, e.g. we write $J$ instead of $J(z)$. 

\section{Pathlength-optimal control} \label{control-sec}
We obtain a computationally efficient state-space description of a controller satisfying condition (\ref{control-regret-gamma-cond}):
\begin{theorem} \label{pathlength-optimal-controller-thm-ih}
Suppose $(A, B_u)$ is stabilizable and $(A, Q^{1/2})$ is detectable. A causal infinite-horizon controller satisfying (\ref{control-regret-gamma-cond}) exists if and only if there exists a solution to the Ricatti equation
\begin{equation*}
\hat{P} = \hat{L}^*\hat{L} + \hat{A}^*\hat{P}\hat{A} -\hat{A}^* \hat{P} \tilde{B}\tilde{H}^{-1}\tilde{B}^*\hat{P}\hat{A}
\end{equation*}
where $$\hat{L} =  \begin{bmatrix} L & 0 \end{bmatrix}, \hspace{3mm} \hat{A} =  \begin{bmatrix} A & -B_wK_2  \\ 0 & \tilde{A} - \tilde{B}_wK_2  \end{bmatrix} \hspace{3mm} \hat{B}_w = \begin{bmatrix} B_w \Sigma_2^{-1/2} \\ \tilde{B}_w \Sigma_2^{-1/2} \end{bmatrix}, \hspace{3mm} \hat{B}_u = \begin{bmatrix} B_u \\ 0 \end{bmatrix},$$
$$\tilde{B} = \begin{bmatrix} \hat{B}_{u} & \hat{B}_{w} \end{bmatrix}, \hspace{3mm} \tilde{R} = \begin{bmatrix} I & 0 \\ 0 & - I \end{bmatrix}, \hspace{3mm} \tilde{H} =  \tilde{R}  + \tilde{B}^* P \tilde{B}, $$
where $\tilde{A}, \tilde{B}_w$ are defined in (\ref{tilde-A-tilde-B-def}) and $K_2, \Sigma_2$ are defined in (\ref{K2-Sigma2-def}), such that 
\begin{enumerate}
    \item $ \hat{A} - \tilde{B}\tilde{H}^{-1}\tilde{B}^*\hat{P}\hat{A}$ is stable;
    \item $\tilde{R}$ and $\tilde{H}$ have the same inertia;
    \item $\hat{P} \succeq 0$.
\end{enumerate}
In this case, a causal infinite-horizon $H_{\infty}$ controller satisfying (\ref{control-regret-gamma-cond}) is given by
\begin{equation*}
u_t =  -R^{-1/2}\hat{H}^{-1} \hat{B}_{u}^*\hat{P}(\hat{A}\xi_t +  \hat{B}_w  w_t'),
\end{equation*}
where $\hat{H} = I + \hat{B}_u^*\hat{P}\hat{B}_u$ and the dynamics of $\xi$ are
\begin{equation*}
\xi_{t+1} = \hat{A} \xi_t + \hat{B}_{u} u_t + \hat{B}_{w} w_t' 
\end{equation*}
and we initialize $\xi_0 = 0$. The synthetic disturbance $w'$ can be computed using the recursion
$$\nu_{t+1} = \tilde{A}\nu_t + \tilde{B}_w w_t, \hspace{3mm} w_t' = \Sigma_2^{1/2} (K_2\nu_t + w_t),$$ A strictly causal infinite-horizon controller satisfying (\ref{control-regret-gamma-cond}) exists if and only if conditions 1 and 3 hold, and additionally $$\hat{B}_u^*\hat{P}\hat{B}_u \prec \gamma^2 I, \hspace{3mm} I + \hat{B}_w^*\hat{P}(I - \hat{B}_u(-\gamma^2 I + \hat{B}_u^* \hat{P} \hat{B}_u)^{-1}\hat{B}_u^* \hat{P})\hat{B}_w \succ 0.$$
In this case, a strictly causal controller satisfying (\ref{control-regret-gamma-cond}) is given by \begin{equation*}
u_t =  -R^{-1/2}\hat{H}^{-1} \hat{B}_{u}^*\hat{P}\hat{A}\xi_t.
\end{equation*}
\end{theorem}

\begin{proof}
Taking the $z$-transform of the linear evolution equations $$x_{t+1} = A x_t + B_uR^{-1/2} v_t + B_w w_t, \hspace{3mm} s_t = L x_t$$ we obtain $$zx(z) = Ax(z) + B_u R^{-1/2} v(z) +  B_w w(z), \hspace{3mm} s(z) = L x(z)$$ Letting $F(z)$ and $G(z)$ be the transfer matrices mapping $v(z)$ and $w(z)$ to $s(z)$, respectively, we see that $$F(z) = L(zI - A)^{-1}B_u R^{-1/2}, \hspace{3mm} G(z) = L(zI - A)^{-1}B_w.$$ 
Using the offline optimal cost (\ref{offline-operator-cost}) and the fact that $\textsc{pathlength}(w) = \|Mw(z)\|_2^2$, where $M(z) = 1-z$, we see that condition (\ref{control-regret-gamma-cond}) is equivalent to  $$ALG(w) < w^*\left[\gamma^2 M^* M + G^*\left(I + FF^*\right)^{-1}G \right]w,$$ where we have suppressed the dependence on $z$ to improve legibility.  Our goal is to obtain a canonical factorization 
\begin{equation} \label{control-factorization}
\gamma^2 M^* M + G^*\left(I + FF^*\right)^{-1}G = \Delta^* \Delta(z).
\end{equation}

With this factorization, condition (\ref{control-regret-gamma-cond}) becomes the $H_{\infty}$ condition $$ ALG(w) < \|w'\|_2^2,$$ where the synthetic disturbance $w'$ is $w'(z) = \Delta(z)w(z)$ and the system dynamics in the frequency domain are 
\begin{equation} \label{freq-domain-dynamics-delta}
s(z) = F(z)v(z) + G(z)\Delta^{-1}(z)w'(z).
\end{equation}
This establishes that the pathlength-optimal controller is the $H_{\infty}$ optimal controller in the system (\ref{freq-domain-dynamics-delta}).

We first factor $I + F(z) F(z^{-*})^*$ as $\Delta_1(z) \Delta_1(z^{-*})^*$. Applying Lemma 1, we see that
\begin{equation} \label{delta-definition}
\Delta_1(z) = (I + L(zI - A)^{-1}K_1 )\Sigma_1^{1/2}.
\end{equation}
We have $$ \Delta_1^{-1}(z) = \Sigma_1^{-1/2} \left(I - L(zI - (A - K_1L))^{-1}K_1 \right), $$ $$G(z) = L(zI - A)^{-1}B_w,$$ therefore a minimal representation of $\Delta_1^{-1}(z)G(z)$ is $$\Delta_1^{-1}(z)G(z) = \Sigma_1^{-1/2} L(zI - (A - K_1L))^{-1}B_w.$$

We can now recover the factorization (\ref{control-factorization}).
Notice that the left-hand side of (\ref{control-factorization}) can be written as $$\begin{bmatrix} \tilde{B}_w^*(z^{-*}I - \tilde{A})^{-*} & I  \end{bmatrix} \begin{bmatrix} \tilde{L}^*\tilde{L} & \tilde{S} \\ \tilde{S}^* & \gamma^2 I \end{bmatrix} \begin{bmatrix} (zI - \tilde{A})^{-1}\tilde{B}_w \\ I \end{bmatrix},$$
 where we define 
\begin{equation} \label{tilde-A-tilde-B-def}
 \tilde{L} = \begin{bmatrix} \Sigma_1^{-1/2}L & 0 \\ 0 & \gamma I \end{bmatrix}, \hspace{3mm} \tilde{S} = \begin{bmatrix} 0 \\ \gamma^2 I \end{bmatrix}, \hspace{3mm} \tilde{A} = \begin{bmatrix} A - K_1L & 0 \\ 0 & 0 \end{bmatrix}, \hspace{3mm} \tilde{B}_w = \begin{bmatrix} B_w \\  -I \end{bmatrix}.
\end{equation}
Applying Lemma 5, we see that this equals $$\begin{bmatrix} \tilde{B}_w^*(z^{-*}I - \tilde{A})^{-*} & I \end{bmatrix} \Lambda_2(P_2) \begin{bmatrix} (zI - \tilde{A})^{-1}\tilde{B}_w \\ I \end{bmatrix},$$
where $P_2$ is an arbitrary Hermitian operator and we define $$\Lambda_2(P_2) = \begin{bmatrix} \tilde{L}^*\tilde{L} - P_2 +  \tilde{A}^*P_2\tilde{A} & \tilde{S} + \tilde{A}^*P_2\tilde{B}_w \\ \tilde{S}^* + \tilde{B}_w^*P_2\tilde{A} & \gamma^2I + \tilde{B}_w^*P_2\tilde{B}_w \end{bmatrix}.$$ Notice that the $\Lambda_2(P_2)$ can be factored as $$\begin{bmatrix} I & K_2^*(P_2) \\ 0 & I \end{bmatrix} \begin{bmatrix} \Gamma_2(P_2) & 0 \\ 0 & \Sigma_2(P_2) \end{bmatrix} \begin{bmatrix} I & 0 \\ K_2(P_2) & I \end{bmatrix}, $$
where we define $$\Gamma_2(P_2) = \tilde{L}^*\tilde{L} - P_2 +  \tilde{A}^*P_2\tilde{A} - K_2^*(P_2)\Sigma_2K_2(P_2),$$ 
\begin{equation} \label{K2-Sigma2-def}
K_2(P_2) = \Sigma_2^{-1}(P_2)(\tilde{S}^* + \tilde{B}_w^*P_2\tilde{A}), \hspace{3mm} \Sigma_2(P_2) = \gamma^2I + \tilde{B}_w^*P_2\tilde{B}_w.
\end{equation}
It is clear that $(\tilde{A}, \tilde{B}_w)$ is stabilizable, therefore the Riccati equation $\Gamma_2(P_2) = 0$ has a unique stabilizing solution (see, e.g. Theorem E.6.2 in \cite{kailath2000linear}).  Suppose $P_2$ is chosen to be this solution, and define $K_2 = K_2(P_2)$, $\Sigma_2 = \Sigma_2(P_2)$. We immediately obtain the factorization (\ref{control-factorization}), where we define 
\begin{equation} \label{control-delta}
\Delta(z) = \Sigma_2^{1/2}(I + K_2(zI - \tilde{A})^{-1}\tilde{B}_w).
\end{equation}
Recall that the pathlength-optimal controller is the $H_{\infty}$-optimal controller in the system (\ref{freq-domain-dynamics-delta}), driven by the synthetic disturbance $w'(z) = \Delta(z) w(z)$. We have $$\Delta^{-1}(z) = (I - K_2(zI - (\tilde{A} - \tilde{B}_wK_2))^{-1}\tilde{B}_w) \Sigma_2^{-1/2}. $$ We note that $ \tilde{A} - \tilde{B}_wK_2$ is stable and hence $\Delta^{-1}(z)$ is causal and bounded since its poles are strictly contained in the unit circle. Define $$\hat{L} =  \begin{bmatrix} L & 0 \end{bmatrix}, \hspace{3mm} \hat{A} =  \begin{bmatrix} A & -B_wK_2  \\ 0 & \tilde{A} - \tilde{B}_wK_2  \end{bmatrix} \hspace{3mm} \hat{B}_w = \begin{bmatrix} B_w \Sigma_2^{-1/2} \\ \tilde{B}_w \Sigma_2^{-1/2} \end{bmatrix}, \hspace{3mm} \hat{B}_u = \begin{bmatrix} B_u \\ 0 \end{bmatrix}.$$
It is easy to verify that $$F(z) = \hat{L}(zI - \hat{A})^{-1}\hat{B}_u, \hspace{3mm} G(z)\Delta_1^{-1}(z) = \hat{L}(zI - \hat{A})^{-1} \hat{B}_w.$$ We have shown that the pathlength-optimal controller in the linear dynamical system parameterized by $\{A, B_uR^{-1/2}, B_w, L\}$ is the $H_{\infty}$-optimal controller in the system parameterized by $\{\hat{A}, \hat{B}_u, \hat{B}_w, \hat{L} \}$. Plugging these parameters into Theorem \ref{hinf-ih-controller-thm} immediately yields the pathlength-optimal control action $v_t$; we recover $u_t$ by setting $u_t = R^{-1/2}v_t$. A state-space model for the synthetic disturbance $w'(z) = \Delta(z)w(z)$ is easily obtained from (\ref{control-delta}).
\end{proof}

\section{Pathlength-optimal filtering} \label{filtering-sec}

We obtain a computationally efficient state-space description of a filter satisfying condition (\ref{filter-regret-gamma-cond}):

\begin{theorem}
A filter satisfying (\ref{filter-regret-gamma-cond}) exists if and only if $$\bar{\sigma}(Z\Pi) \leq 1,$$
where $Z$ and $\Pi$ are solutions of the Lyapunov equations $$Z = F^*ZF + H^*H, \hspace{3mm} \Pi = F \Pi F^* + GG^* $$
and $F, G, H$ are defined in (\ref{fgh-def}). In this case, a filtered estimator satisfying (\ref{filter-regret-gamma-cond}) is given by $$\hat{s}_t = \Sigma_3^{-1/2}\beta_t - K_3 \pi_t, $$ where $\beta, \pi$ have the dynamics $$\pi_{t+1} = (A_1 - A_1W_1L^*K_3)\pi_t + A_1W_1L^*\Sigma_3^{-1/2}\beta_t, \hspace{3mm}  \beta_t = \Delta_3(1)Q(1)z_t + \alpha_t - \alpha_{t-1}.$$ The variable $\alpha$ is given by $$\alpha_t = H(\Pi\xi^1_{t+1} + \xi^2_t),$$ where $\xi^1, \xi^2$ have dynamics $$\xi^1_{t+1} = F_{\gamma}\xi^1_t + K_{\gamma}z_t, \hspace{3mm} \xi^2_{t+1} = A_2^*\xi^2_t + C^*\Sigma_2^{-1/2}z_t$$ where $$z_t = \Sigma_2^{-1/2}y_t - Ce_t, \hspace{3mm} e_{t+1} = A_2e_t + K_2y_t.$$
The matrices $\Delta_3(1)Q(1), A_1, A_2, \Sigma_2, \Sigma_3, K_2, K_3$ and $W_1$ are defined as in Lemmas 1, 2, and 3, and $F_{\gamma}, K_{\gamma}$ are defined in (\ref{fgamma-kgamma-def}).
\end{theorem}

\begin{proof}
Our goal is to find a filter $K$ satisfying (\ref{filter-regret-gamma-cond}). Moving to the  $z$-domain, this condition can be neatly expressed in terms of transfer matrices as
\begin{equation} \label{regret-cond-transfer-operator}
0 \prec \gamma^2 M^*M + T_K^*T_K - T_{K_0}^*T_K,
\end{equation}
where we have suppressed the dependence on $z$ to improve legibility and defined $$M(z) = \begin{bmatrix} I_m & 0 \\ 0 &  M_v(z)\end{bmatrix} $$ and $M_v(z) = (1-z^{-1})I_p$. 
Define the unitary operator $$\theta = \begin{bmatrix} I & H^* \\ -H & I \end{bmatrix} \begin{bmatrix} \Delta_1^{-1} & 0 \\ 0 & -\Delta_2^{-*} \end{bmatrix},$$
where we define causal operators $\Delta_1, \Delta_2$ such that $$ \Delta_1^*\Delta_1 = I + H^*H, \hspace{3mm} \Delta_2 \Delta_2^* = I + HH^*. $$
Notice that for every estimator $K$, we have $$\mathcal{T}_{K}\theta =  \begin{bmatrix} J\Delta_1^{-1} & K\Delta_2 - JH^*\Delta_2^{-*}  \end{bmatrix},$$ and in particular, we have $$\mathcal{T}_{K_0}\theta =  \begin{bmatrix} J\Delta_1^{-1} & 0 \end{bmatrix}.$$
Since $\theta$ is unitary and commutes with $M(z)$, condition (\ref{regret-cond-transfer-operator}) is equivalent to 
\begin{equation} \label{regret-cond-transfer-operator-2}
0 \prec \gamma^2 M^*M + \theta^*\mathcal{T}_{K_0}^* \mathcal{T}_{K_0}\theta - \theta^*\mathcal{T}_{K}^* \mathcal{T}_{K}\theta.
\end{equation}
Define $$ R(z) = J(z)\Delta_1^{-1}(z), \hspace{3mm} Q(z) = J(z)H^*(z^{-*}))\Delta_2^{-*}(z^{-*}).$$
The right-hand side of (\ref{regret-cond-transfer-operator-2}) can be rewritten as 
\begin{equation*} 
\begin{bmatrix} \gamma^2 I_m & -R^* (K\Delta_2 - Q) \\ -(K\Delta_2 - Q)^*R & \gamma^2 M_v^*M_v - (K\Delta_2 - Q)^*(K\Delta_2 - Q)  \end{bmatrix}
\end{equation*}
where we suppress the dependence on $z$ to improve legibility.
Applying the Schur complement, we see that condition (\ref{regret-cond-transfer-operator}) is equivalent to
$$(K\Delta_2 - Q)^*(\gamma^{-2}I_m + \gamma^{-4}RR^*) (K\Delta_2 - Q) \preceq  M_v^*M_v.$$
Suppose $\gamma^{-2}I_m + \gamma^{-4}RR^*$ can be factored as $\Delta_3^*(z^{-*})\Delta_3(z)$. where $\Delta_3(z)$ is causal. Then this condition is equivalent to \begin{equation} \label{rhs-regret-cond-transfer-operator-3}
(\Delta_3K\Delta_2 - \Delta_3 Q)^*(\Delta_3K\Delta_2 - \Delta_3Q) \preceq M_v^*M_v. 
\end{equation}
In order for condition (\ref{filter-regret-gamma-cond}) to hold, we must pick $K$ such that $\Delta_3(1)K(1)\Delta_2(1) = \Delta_3(1)Q(1)$; otherwise the regret incurred when $w = 0$ and $v$ is constant would be nonzero even though the energy in $w$ and the pathlength of $v$ are both zero. Notice that   $$\Delta_3(z)K(z)\Delta_2(z) - \Delta_3(z)Q(z) = (1 - z^{-1})(\tilde{K}(z)  - \tilde{Q}(z)), $$ where $\tilde{K}(z),  \tilde{Q}(z)$ satisfy $$\Delta_3(z)K(z)\Delta_2(z) = \Delta_3(1)K(1)\Delta_2(1)  + (1 - z^{-1})\tilde{K}(z)$$ $$ \Delta_3(z)Q(z) = \Delta_3(1)Q(1) + (1 - z^{-1})\tilde{Q}(z).$$ 
It is easy to see that $K$ is causal if and only if $\tilde{K}$ is causal. Therefore condition (\ref{rhs-regret-cond-transfer-operator-3}) is equivalent to \begin{equation} \label{rhs-regret-cond-transfer-operator-4}
(\tilde{K}(z)  - \tilde{Q}(z))^* (\tilde{K}(z)  - \tilde{Q}(z)) \preceq I.
\end{equation}
By Lemma 3, we can decompose $\Delta_3(z)Q(z)$ as 
$$\hat{L}W_2C^*\Sigma_2^{-1/2} + \hat{L}(zI - \hat{A})^{-1}\hat{A}W_2
C^*\Sigma_2^{-1/2} + \hat{L}W_2A_2^*(z^{-1}I - A_2^*)^{-1} C^*\Sigma_2^{-1/2},$$
where $\hat{L}, \hat{A}, A_2$ and $W_2$ are defined in the lemma and $A_2$ is stable. Therefore $$\Delta_3(z)Q(z) - \Delta_3(1)Q(1) = (1 - z^{-1}) \tilde{Q}(z),$$ where $$\tilde{Q}(z) = S(z) + T(z),$$
and $$S(z) = -z\hat{L}(I-\hat{A})^{-1}(zI - \hat{A})^{-1}\hat{A}W_2
C^*\Sigma_2^{-1/2}$$ is causal and $$T(z) = \hat{L}W_2A_2^*(I-A_2^*)^{-1}(z^{-1}I - A_2^*)^{-1}C^*\Sigma_2^{-1/2} $$ is strictly anticausal. Define $$\hat{K}(z) = \tilde{K}(z) - S(z).$$ It is clear that $\hat{K}(z)$ is causal if and only if $\tilde{K}(z)$ is causal, and furthermore, condition (\ref{rhs-regret-cond-transfer-operator-4}) is equivalent to 
\begin{equation} \label{rhs-regret-cond-transfer-operator-5}
 \|\hat{K}(z)  - T(z)\|_{\infty} < 1.
\end{equation}
We recognize (\ref{rhs-regret-cond-transfer-operator-5}) as a Nehari problem.
Define 
\begin{equation} \label{fgh-def}
H = \hat{L}W_2A_2^*(I-A_2^*)^{-1}, \hspace{2mm} F =  A_2^*, \hspace{2mm} G = C^*\Sigma_2^{-1/2}. 
\end{equation}
Then $$T(z) = H(z^{-1}I - F)^{-1}G,$$ where $F$ is stable. 
Appealing to Theorem 2 in the main text, we see that the smallest possible value of $\gamma$ such that a solution to (\ref{rhs-regret-cond-transfer-operator-5}) exists is the smallest value of $\gamma$ such that $ \bar{\sigma}(Z\Pi) \leq 1$,
where $Z$ and $\Pi$ are solutions of the Lyapunov equations $$Z = F^*ZF + H^*H, \hspace{3mm} \Pi = F \Pi F^* + GG^*. $$ Let $\gamma^{\star}$ be this optimal value of $\gamma$. For all $\gamma \geq \gamma^*$, a corresponding $\gamma$-optimal solution $\hat{K}(z)$ to (\ref{rhs-regret-cond-transfer-operator-5})  is $$\hat{K}(z) = H\Pi(I + F_{\gamma}(zI - F_{\gamma})^{-1})K_{\gamma}, $$
where 
\begin{equation} \label{fgamma-kgamma-def}
F_{\gamma} = F^* - K_{\gamma}G^*, \hspace{3mm} K_{\gamma} = (I - F^* Z F \Pi)^{-1} F^* Z G.
\end{equation}
We hence obtain the following expression for $K(z)$ in terms of $\hat{K}(z)$: $$ \Delta_3^{-1}(z)\left(\Delta_3(1)Q(1) + (1-z^{-1})(\hat{K}(z) + T(z) )\right)\Delta_2^{-1}(z). $$ 
This describes a filter satisfying condition (\ref{filter-regret-gamma-cond}) in the $z$-domain; translating this transfer matrix back into time domain, we easily obtain the state-space model described in the theorem. The constant $\Delta_3(1)Q(1)$ is easily found from the definitions of $\Delta_3(z)$ and $Q(z)$.

\end{proof}

\section{Experiments} \label{experiments-sec}
\subsection{Pathlength-optimal control}
We benchmark our pathlength-optimal controller in a nonlinear inverted pendulum system. This dynamical system has two scalar states, $\theta$ and $\dot{\theta}$, representing angular position and angular velocity, respectively, and a single scalar control input $u$. The state $(\theta, \dot{\theta})$ evolves according to the nonlinear evolution equation 
$$\frac{d}{dt} \begin{bmatrix} \theta \\ \dot{\theta} \end{bmatrix} = \begin{bmatrix} \dot{\theta} \\ \frac{mg\ell}{J}\sin{\theta} + \frac{\ell}{J}u\cos{\theta} + \frac{\ell}{J}w\cos{\theta} \end{bmatrix},$$ 
where $w$ is an external scalar disturbance, and $m, \ell, g, J$ are physical parameters describing the system; we assume that units are scaled so that these parameters are 1. Although these dynamics are nonlinear, we can benchmark the regret-optimal controller against the $H_2$-optimal, $H_{\infty}$-optimal, and clairvoyant offline optimal controllers using Model Predictive Control (MPC). In the MPC framework, we iteratively linearize the model dynamics around the current state, compute the optimal control action in the linearized system, and then update the state in the original nonlinear system using this control action. In all of our experiments we take $Q, R = I$ and initialize $\theta$ and $\dot{\theta}$ to zero. We set the discretization parameter $\delta_t = 0.001$ and sample the dynamics at intervals of $\delta_t$.

In Figure \ref{control-plots}, we plot the relative performance of the pathlength-optimal, $H_2$-optimal, and offline optimal controllers across different input disturbances. In several of our experiments the cost incurred by the $H_{\infty}$-optimal controller is orders of magnitude higher than that of the other controllers and is not shown in the corresponding figure. We first take $w$ to be sampled i.i.d from a standard Gaussian in Figure \ref{invpend-gaussian-fig}; surprisingly, the pathlength-optimal controller significantly outperforms the $H_2$-optimal controller, which is tuned for i.i.d zero-mean noise. This may be because the pathlength-optimal controller is better able to adapt to nonlinear dynamics (though it is only designed to adapt to the disturbance-generating process, not the dynamics).  In Figure \ref{invpend-step-fig}, the disturbance is a step-function: it is $+1$ for 500 timesteps and $-1$ for the next 500 timesteps. While this disturbance has high energy, its pathlength is only 4. As expected, the pathlength-optimal controller outperforms the other causal controllers. In Figures \ref{invpend-ones-fig}-\ref{invpend-sine-3-fig}, we consider sinusoidal disturbances with amplitude 1 and periods $\infty$ (i.e. a constant disturbance), $200\pi, 20\pi$ and $2\pi$, respectively. Since the sine function is $1$-Lipschitz, these disturbance are respectively 0.0-, 0.01-. 0.1-, and 1-temporally Lipschitz. In all of these experiments, the pathlength-optimal controller outperforms the $H_2$-optimal and $H_{\infty}$-optimal controllers and closely tracks the clairvoyant offline optimal controller.

\begin{figure}[htb]
\centering
\begin{subfigure}{0.45\textwidth}
\includegraphics[width=\columnwidth]{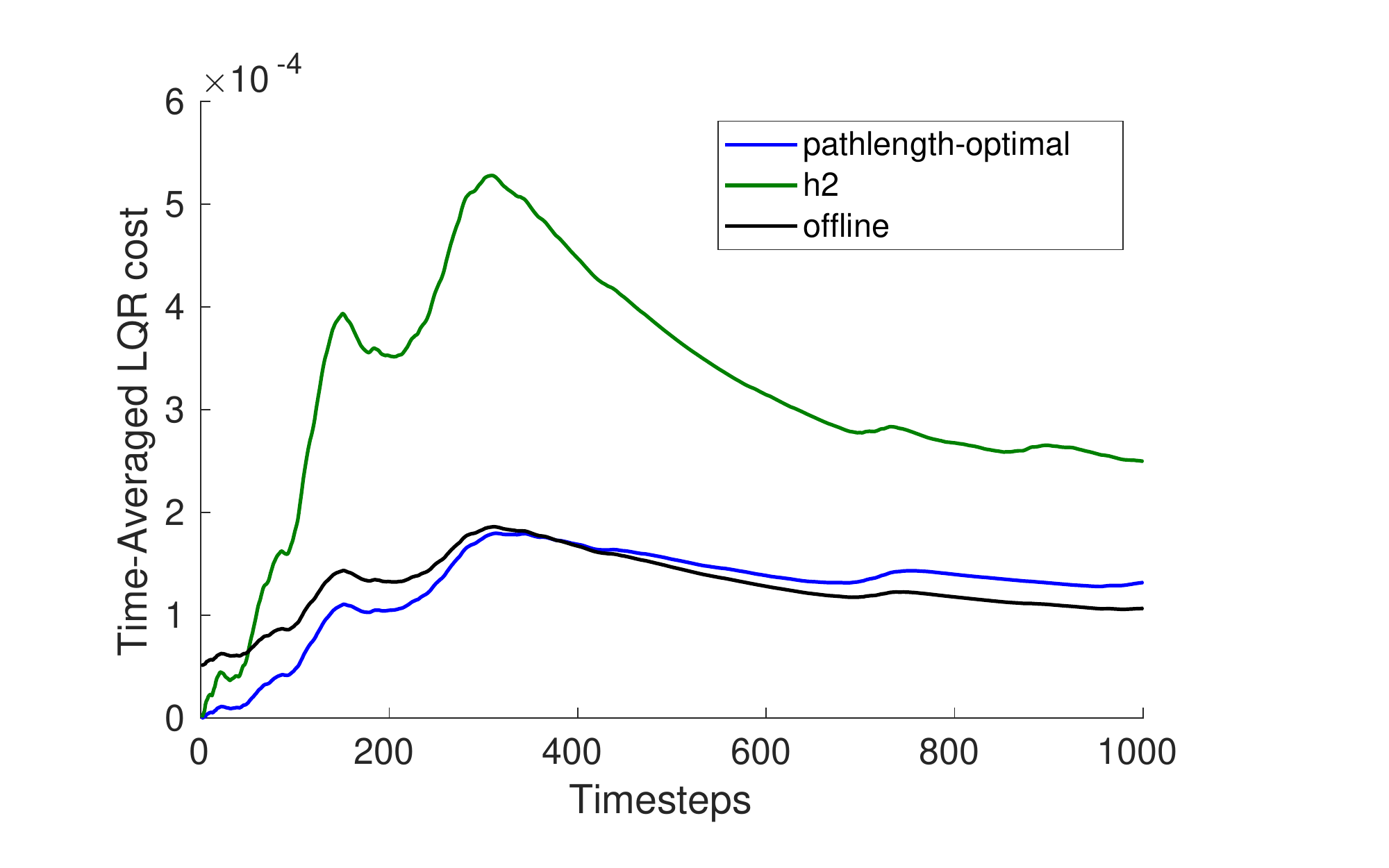}
\caption{The driving disturbance is zero-mean Gaussian noise. }
\label{invpend-gaussian-fig}
\end{subfigure} \hfil
\begin{subfigure}{0.45\textwidth}
\includegraphics[width=\columnwidth]{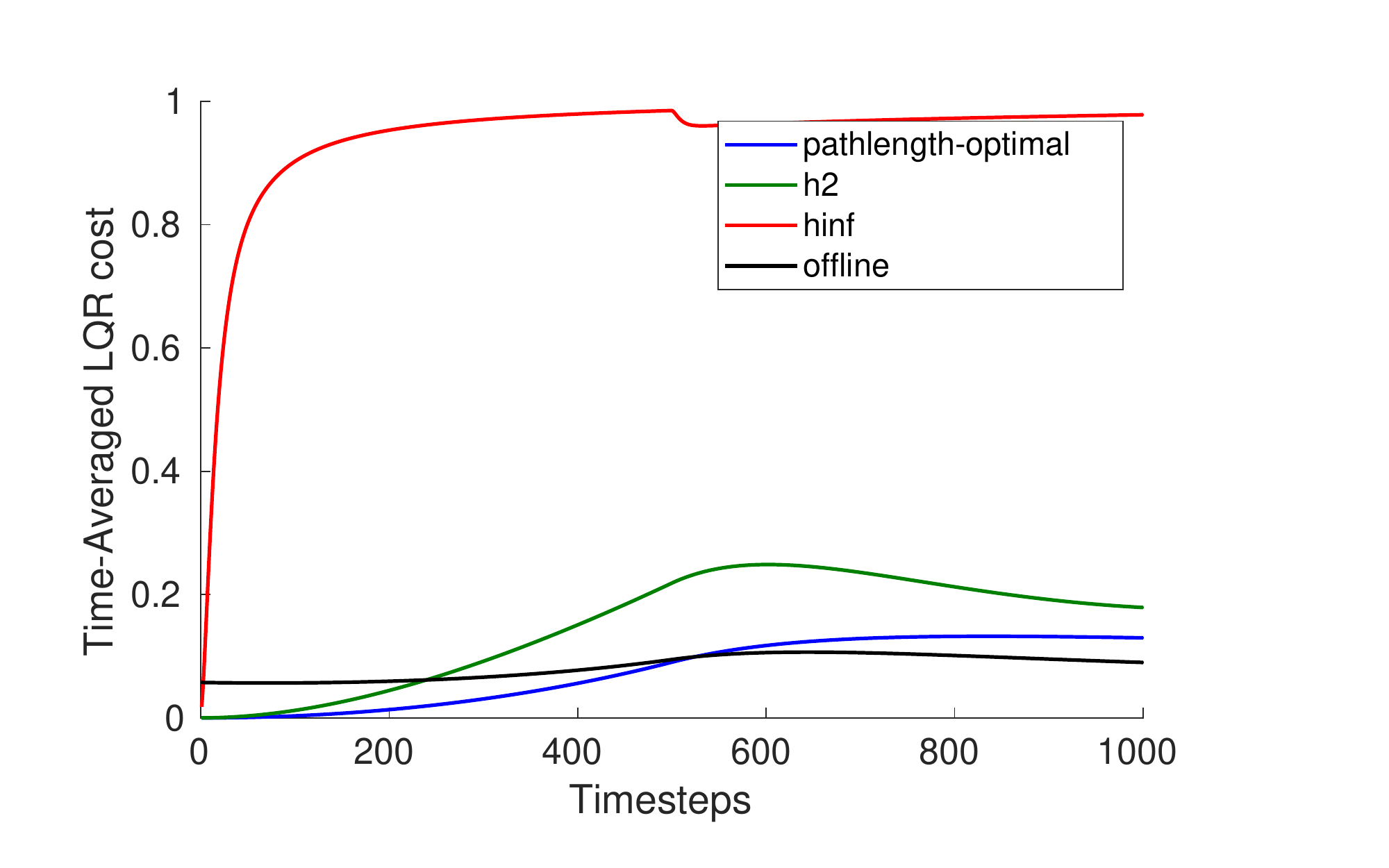}
\caption{The driving disturbance is a step-function. }
\label{invpend-step-fig}
\end{subfigure}
\begin{subfigure}{0.45\textwidth}
\includegraphics[width=\columnwidth]{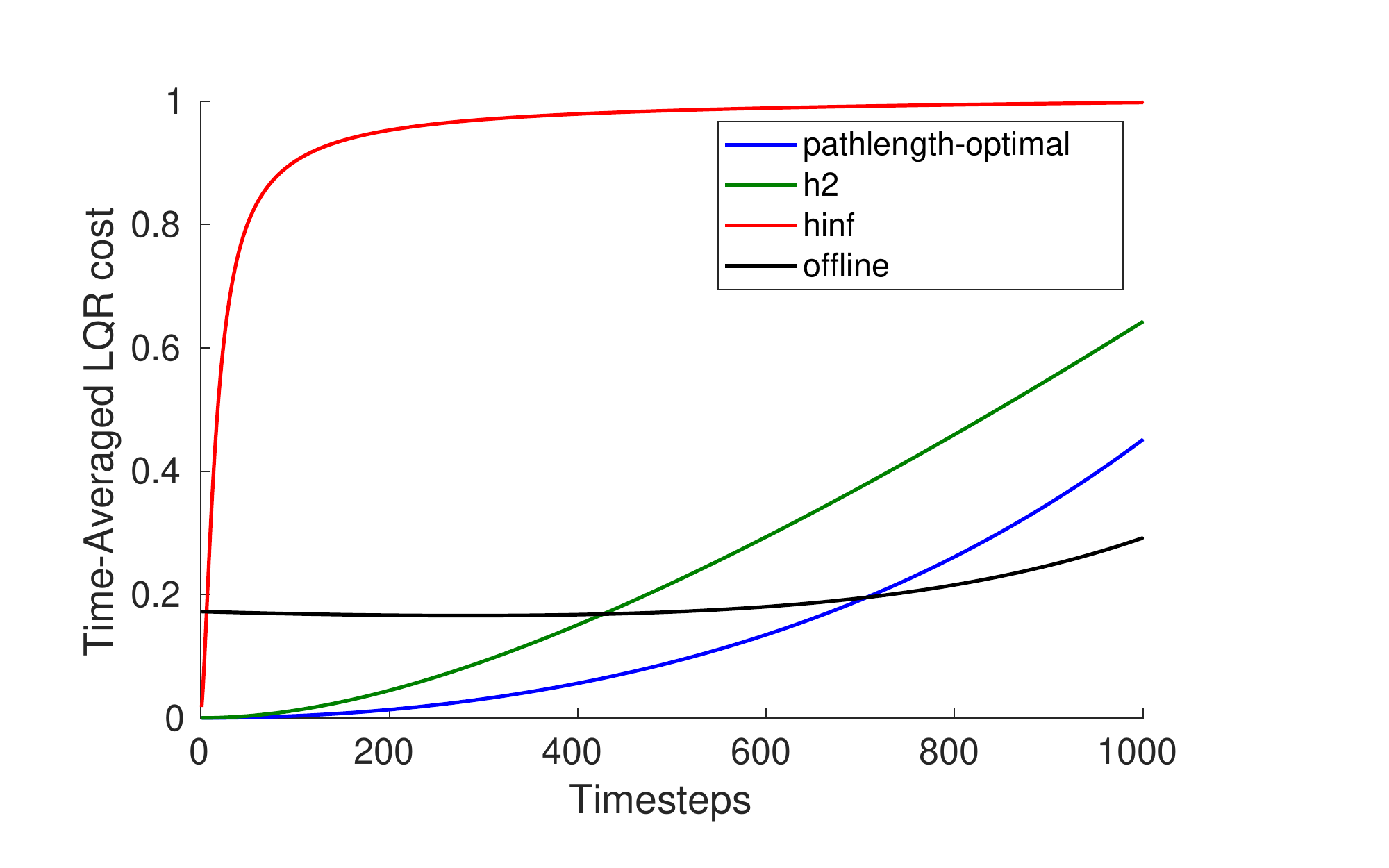}
\caption{The driving disturbance is constant.}
\label{invpend-ones-fig}
\end{subfigure} \hfil
\begin{subfigure}{0.45\textwidth}
\includegraphics[width=\columnwidth]{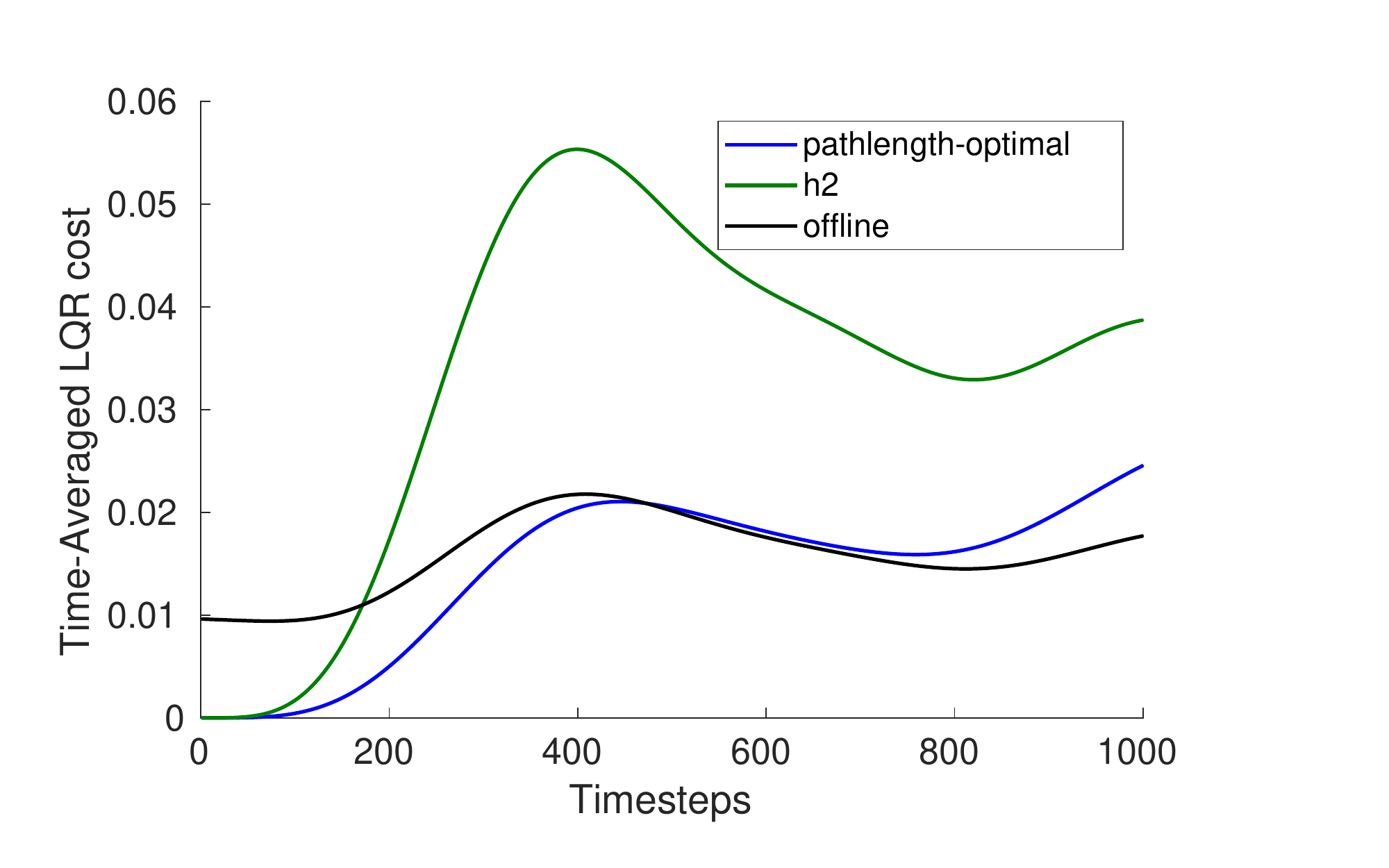}
\caption{The driving disturbance is sinusoidal with period $200\pi$. }
\label{invpend-sine-2-fig}
\end{subfigure}
\begin{subfigure} {0.45\textwidth}
\includegraphics[width=\columnwidth]{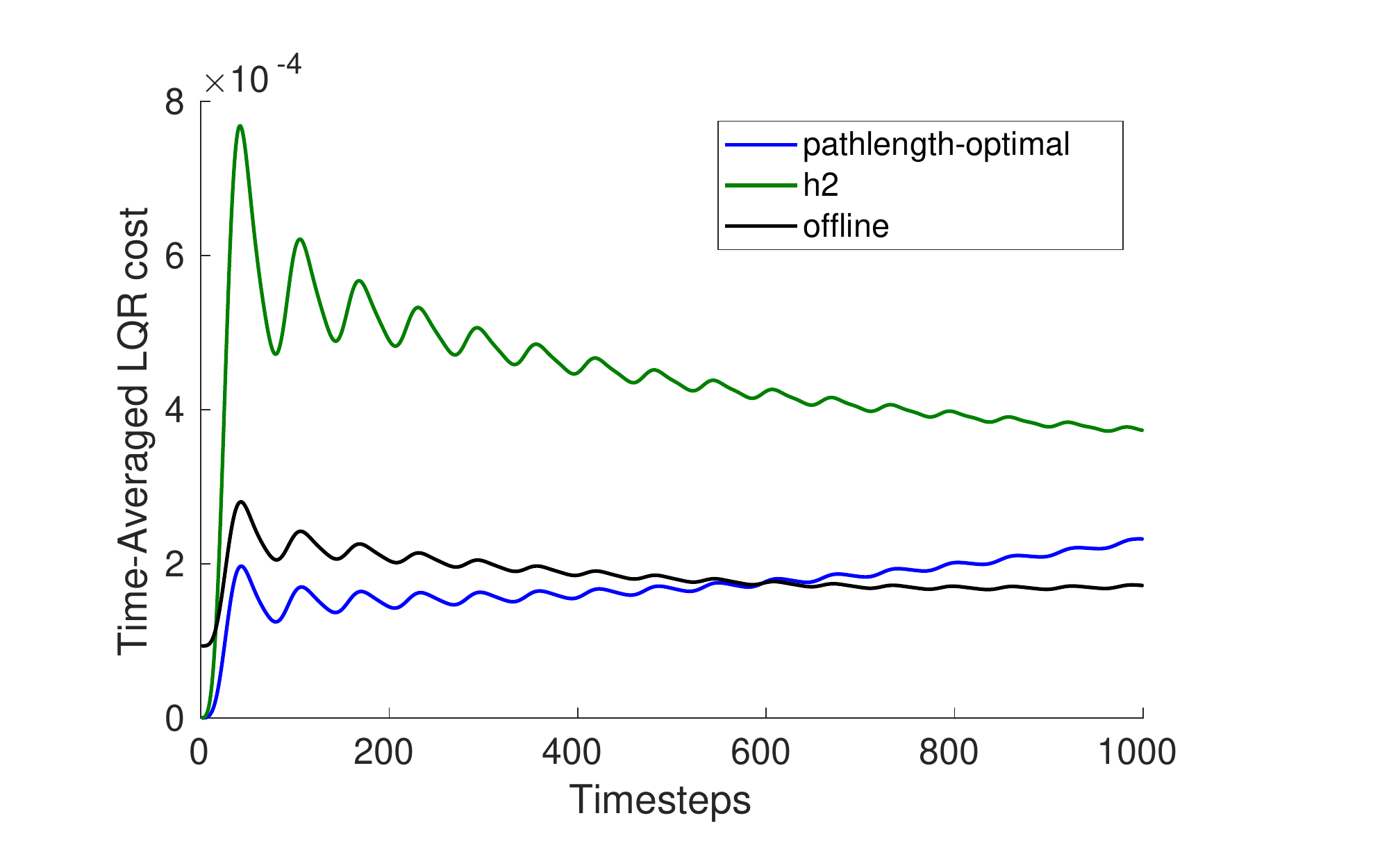}
\caption{The driving disturbance is sinusoidal with period $20\pi$. }
\label{invpend-sine-fig-1}
\end{subfigure} \hfil
\begin{subfigure}{0.45\textwidth}
\includegraphics[width=\columnwidth]{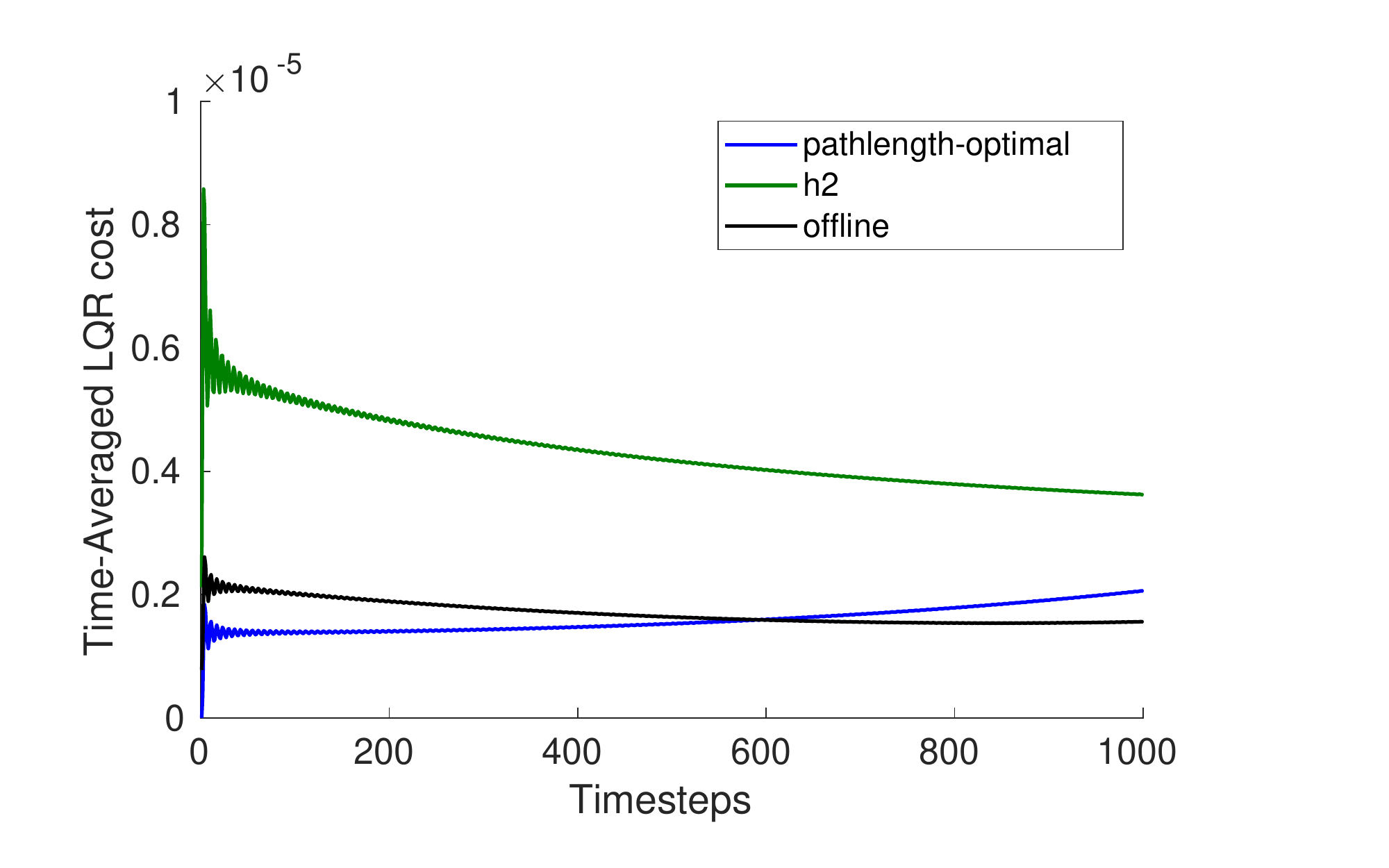}
\caption{The driving disturbance is sinusoidal with period $2\pi$.}
\label{invpend-sine-3-fig}
\end{subfigure}
\caption{Relative performance of linear-quadratic (LQ) controllers.}
\label{control-plots}
\end{figure}

\subsection{Pathlength-optimal filtering}
We consider a one-dimensional tracking problem, where the goal is to estimate an object's position given noisy observations of its trajectory. We consider the state-space model $$\begin{bmatrix} x_{t+1} \\ \nu_{t+1} \end{bmatrix} = \begin{bmatrix} 1 & \delta_t \\ 0 & 1 \end{bmatrix} \begin{bmatrix} x_t \\ \nu_t \end{bmatrix} +  \begin{bmatrix} 0 \\ \delta_t \end{bmatrix} \alpha_t, \hspace{5mm} y_t = x_t + v_t, \hspace{5mm} s_t = x_t,$$ where $x_t$ is the object's position, $\nu_t$ is the object's velocity, $\alpha_t$ is the object's instantaneous acceleration due to external forces, and $v_t$ is measurement noise. We take $\delta_t = 0.01$ and initialize $x_0 = 0$. In this system, the optimal value of $\gamma$ is $\gamma^{\star} \approx 35.64$.

We benchmark the performance of the pathlength-optimal filter against that of the Kalman filter (e.g. the $H_2$-optimal filter).  Recall that the key innovation of the pathlength-optimal filter relative to standard filters is that the pathlength-optimal filter is designed to achieve low regret when the measurement disturbance $v$ has low pathlength; for this reason, we focus on measuring how the performance varies across many different values of $v$. For simplicity, we take the driving disturbance $\alpha$ to be picked i.i.d from a standard Gaussian across all of our experiments. In Figure \ref{ones-fig}, we plot the relative performance of the pathlength-optimal filter against that of the Kalman filter when the measurement disturbance is constant, i.e. $v_t = 1$ for all $t$. This disturbance has high energy but zero pathlength - as expected, the pathlength-optimal filter easily beats the Kalman filter, achieving orders-of-magnitude less estimation error. In Figure \ref{sine-1-fig}, $v$ varies sinusoidally with period $200\pi$. Again, the pathlength-optimal filter outperforms the Kalman filter by orders of magnitude. We next consider a sinusoidal disturbance with period $20\pi$ in Figure \ref{sine-2-fig}. While this disturbance has a higher pathlength, the pathlength-optimal filter is still able to outperform the Kalman filter. Finally, in Figure \ref{sine-3-fig} we consider an adversarial case, where the measurement disturbance $v$ oscillates rapidly and hence has very high pathlength. We see that the pathlength-optimal filter is outperformed by the Kalman filter; this is entirely unsurprising, since the pathlength-optimal filter is specifically designed under the assumption that $v$ varies slowly over time. Together, these four plots are consistent with the message of this paper: when the pathlength of the measurement disturbance is small, the pathlength-optimal filter produces very accurate estimates of the state and outperforms standard filtering algorithms.

\begin{figure}[htb]
\centering
\begin{subfigure}{0.45\textwidth}
\includegraphics[width=\columnwidth]{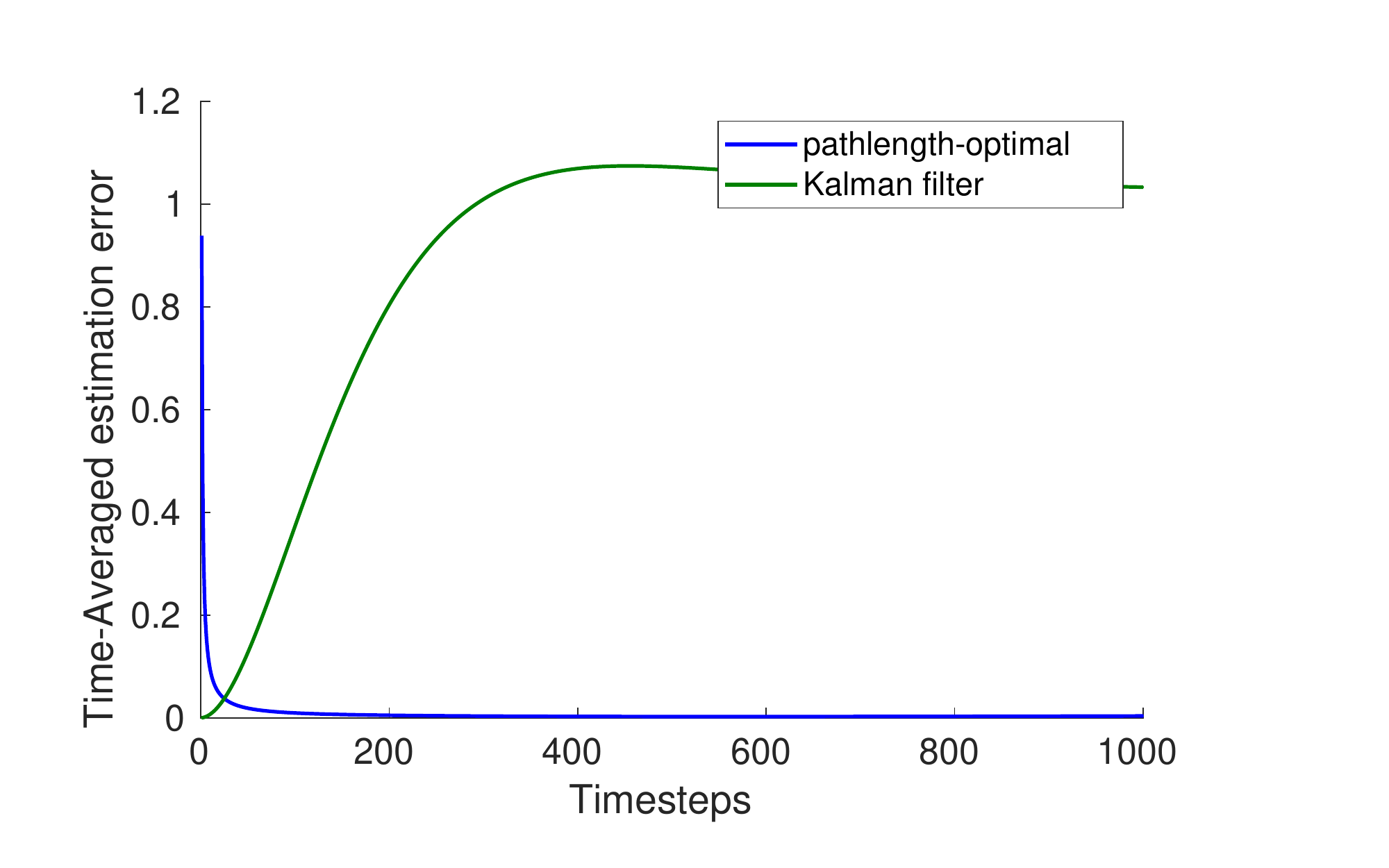}
\caption{The driving disturbance is drawn i.i.d from a standard Gaussian  and the measurement disturbance is constant ($v_t = 1$ for all $t$). The pathlength of the measurement disturbance is zero, though its energy is large.}
\label{ones-fig}
\end{subfigure} \hfil
\begin{subfigure}{0.45\textwidth}
\includegraphics[width=\columnwidth]{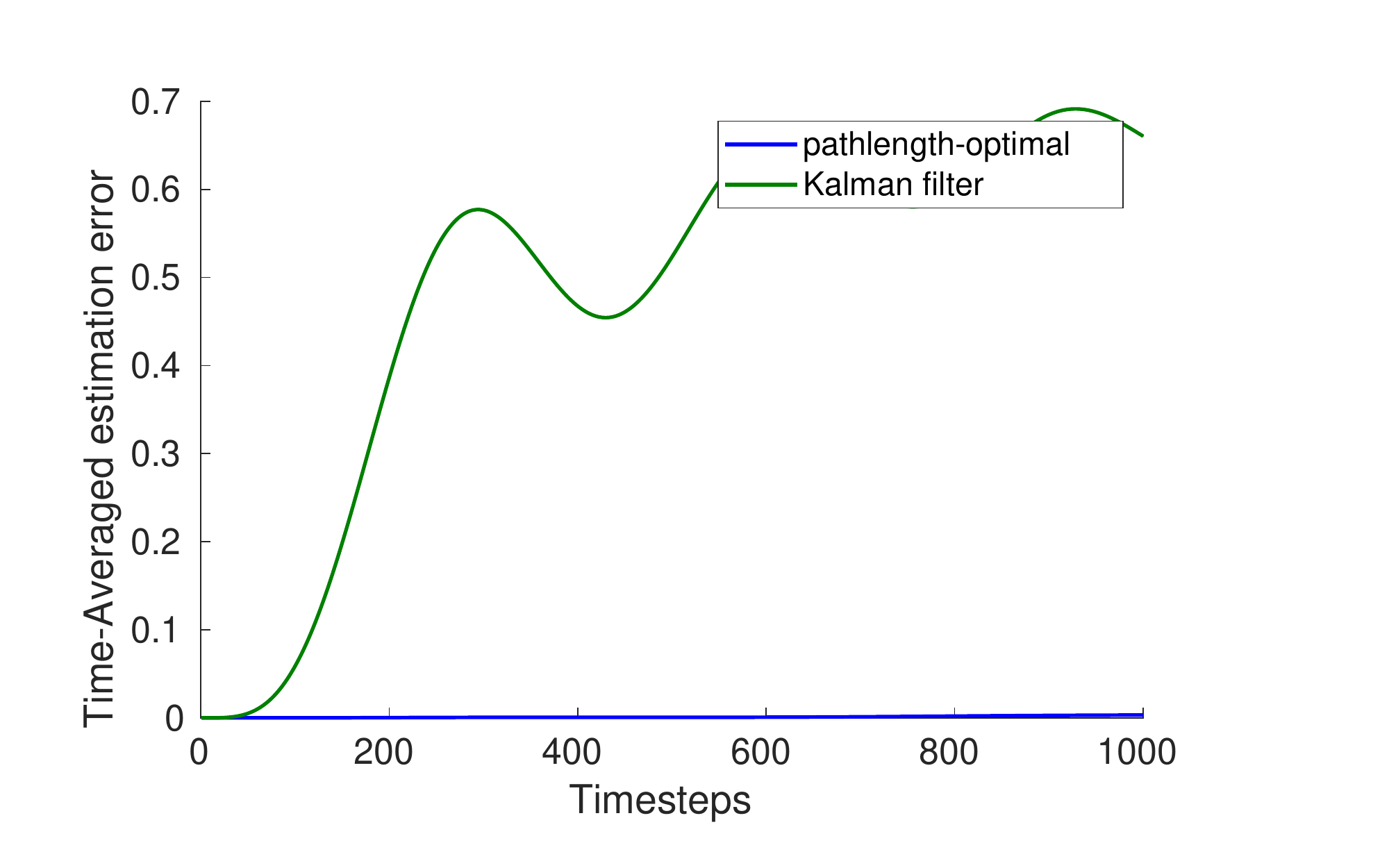}
\caption{The driving disturbance is drawn i.i.d from a standard Gaussian  and the measurement disturbance varies sinusoidally with period $200\pi$. The pathlength of the measurement disturbance is small, relative to its energy.}
\label{sine-1-fig}
\end{subfigure}
\begin{subfigure}{0.45\textwidth}
\includegraphics[width=\columnwidth]{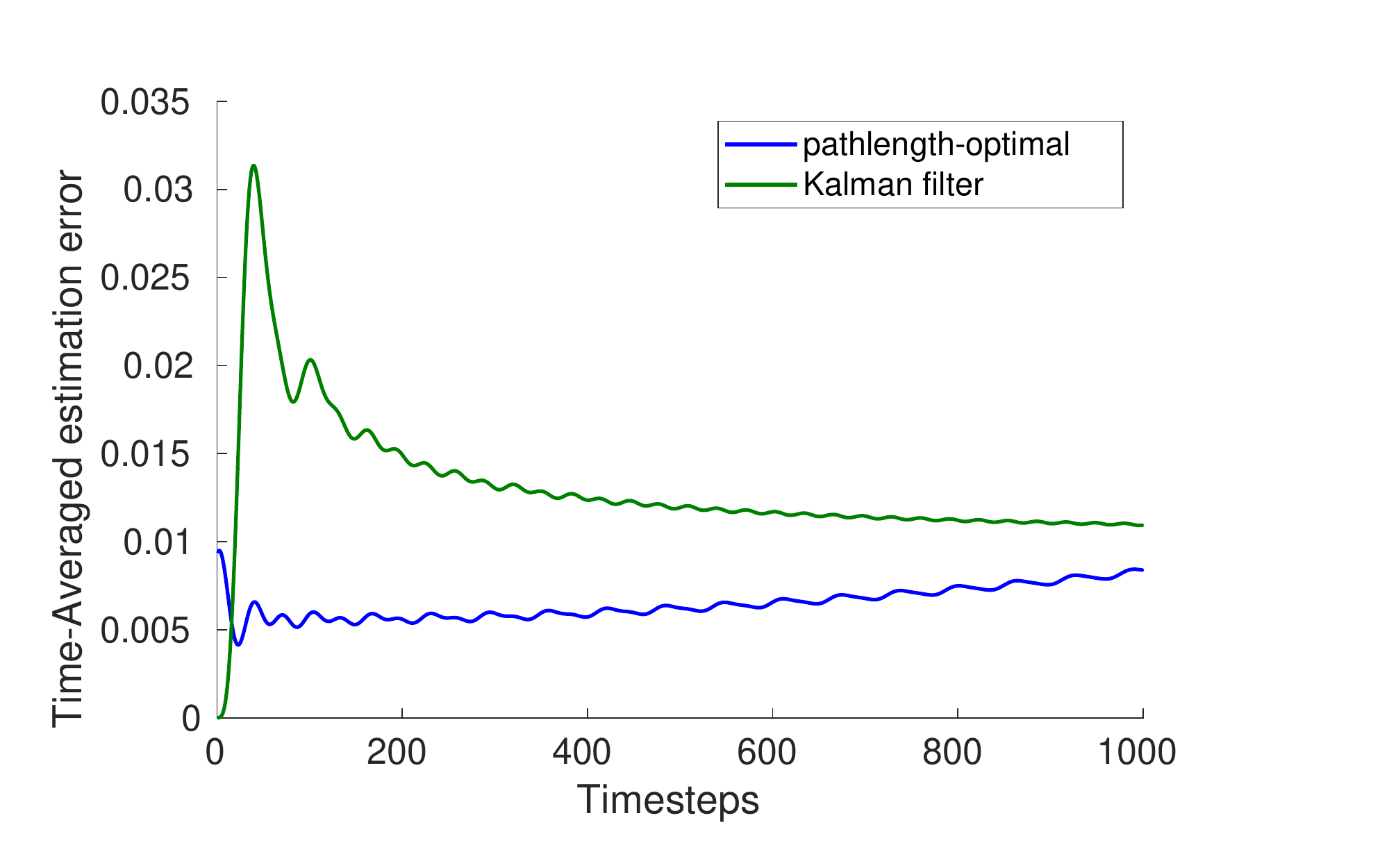}
\caption{The driving disturbance is drawn i.i.d from a standard Gaussian  and the measurement disturbance varies sinusoidally with period $20\pi$. The pathlength of the measurement disturbance is moderate, relative to its energy.}
\label{sine-2-fig}
\end{subfigure} \hfil
\begin{subfigure}{0.45\textwidth}
\includegraphics[width=\columnwidth]{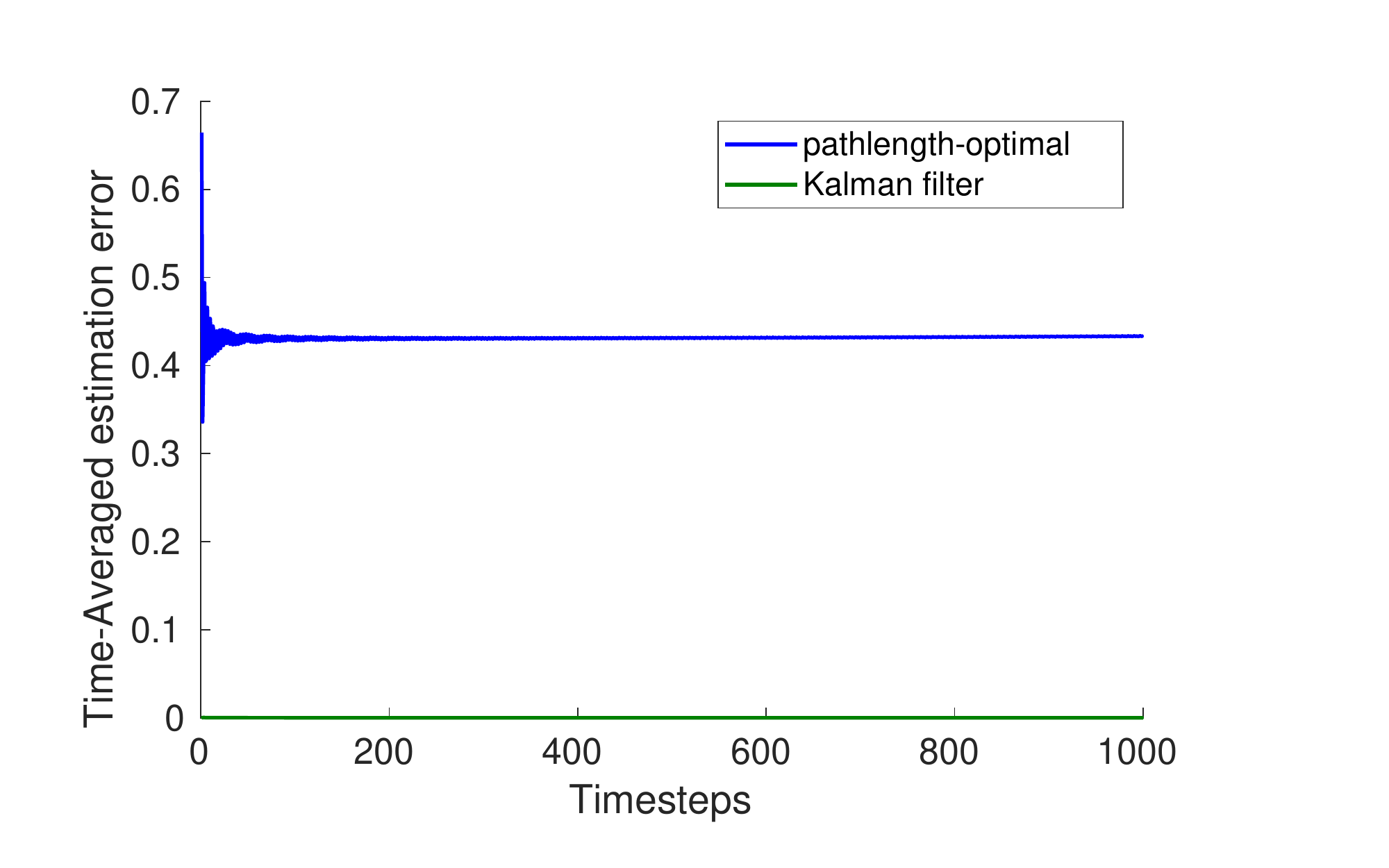}
\caption{The driving disturbance is drawn i.i.d from a standard Gaussian  and the measurement disturbance varies sinusoidally with period $2\pi$. The pathlength of the measurement disturbance is large, relative to its energy.}
\label{sine-3-fig}
\end{subfigure}
\label{filter-plots}
\caption{Relative performance of the pathlength-optimal filter and the Kalman filter.}
\end{figure}

\clearpage
\printbibliography

@article{goel2019online,
  title={An online algorithm for smoothed regression and lqr control},
  author={Goel, Gautam and Wierman, Adam},
  journal={Proceedings of Machine Learning Research},
  volume={89},
  pages={2504--2513},
  year={2019},
  publisher={PMLR}
}

@article{gradu2020adaptive,
  title={Adaptive regret for control of time-varying dynamics},
  author={Gradu, Paula and Hazan, Elad and Minasyan, Edgar},
  journal={arXiv preprint arXiv:2007.04393},
  year={2020}
}

@article{goel2021regret,
  title={Regret-optimal Estimation and Control},
  author={Goel, Gautam and Hassibi, Babak},
  journal={arXiv preprint arXiv:2106.12097},
  year={2021}
}

@inproceedings{goel2021measurement,
  title={Regret-optimal measurement-feedback control},
  author={Goel, Gautam and Hassibi, Babak},
  booktitle={Learning for Dynamics and Control},
  pages={1270--1280},
  year={2021},
  organization={PMLR}
}

@article{zhao2021non,
  title={Non-stationary online learning with memory and non-stochastic control},
  author={Zhao, Peng and Wang, Yu-Xiang and Zhou, Zhi-Hua},
  journal={arXiv preprint arXiv:2102.03758},
  year={2021}
}

@inproceedings{sabag2021regret,
  title={Regret-optimal controller for the full-information problem},
  author={Sabag, Oron and Goel, Gautam and Lale, Sahin and Hassibi, Babak},
  booktitle={2021 American Control Conference (ACC)},
  pages={4777--4782},
  year={2021},
  organization={IEEE}
}

@inproceedings{chiang2012online,
  title={Online optimization with gradual variations},
  author={Chiang, Chao-Kai and Yang, Tianbao and Lee, Chia-Jung and Mahdavi, Mehrdad and Lu, Chi-Jen and Jin, Rong and Zhu, Shenghuo},
  booktitle={Conference on Learning Theory},
  pages={6--1},
  year={2012},
  organization={JMLR Workshop and Conference Proceedings}
}

@article{goel2021competitive,
  title={Competitive Control},
  author={Goel, Gautam and Hassibi, Babak},
  journal={arXiv preprint arXiv:2107.13657},
  year={2021}
}

@article{nehari1957bounded,
  title={On bounded bilinear forms},
  author={Nehari, Zeev},
  journal={Annals of Mathematics},
  pages={153--162},
  year={1957},
  publisher={JSTOR}
}

@inproceedings{sabag2021filtering,
  title={Regret-Optimal Filtering},
  author={Sabag, Oron and Hassibi, Babak},
  booktitle={International Conference on Artificial Intelligence and Statistics},
  pages={2629--2637},
  year={2021},
  organization={PMLR}
}

@article{besbes2015non,
  title={Non-stationary stochastic optimization},
  author={Besbes, Omar and Gur, Yonatan and Zeevi, Assaf},
  journal={Operations research},
  volume={63},
  number={5},
  pages={1227--1244},
  year={2015},
  publisher={INFORMS}
}

@article{wei2016tracking,
  title={Tracking the best expert in non-stationary stochastic environments},
  author={Wei, Chen-Yu and Hong, Yi-Te and Lu, Chi-Jen},
  journal={Advances in neural information processing systems},
  volume={29},
  pages={3972--3980},
  year={2016}
}

@inproceedings{zinkevich2003online,
  title={Online convex programming and generalized infinitesimal gradient ascent},
  author={Zinkevich, Martin},
  booktitle={Proceedings of the 20th international conference on machine learning (icml-03)},
  pages={928--936},
  year={2003}
}

@inproceedings{wei2018more,
  title={More adaptive algorithms for adversarial bandits},
  author={Wei, Chen-Yu and Luo, Haipeng},
  booktitle={Conference On Learning Theory},
  pages={1263--1291},
  year={2018},
  organization={PMLR}
}

@inproceedings{bubeck2019improved,
  title={Improved path-length regret bounds for bandits},
  author={Bubeck, S{\'e}bastien and Li, Yuanzhi and Luo, Haipeng and Wei, Chen-Yu},
  booktitle={Conference On Learning Theory},
  pages={508--528},
  year={2019},
  organization={PMLR}
}

@book{kailath2000linear,
  title={Linear estimation},
  author={Kailath, Thomas and Sayed, Ali H and Hassibi, Babak},
  year={2000},
  publisher={Prentice Hall}
}

@book{hassibi1999indefinite,
  title={Indefinite-quadratic estimation and control: a unified approach to H 2 and H-infinity theories},
  author={Hassibi, Babak and Sayed, Ali H and Kailath, Thomas},
  year={1999},
  publisher={SIAM}
}

\clearpage

\section{Appendix}
We present proofs of some of the key lemmas used in the main body.

\begin{lemma} \label{delta-factorization-lemma}
Let $H(z) = C(zI - A)^{-1}B,$ where $(A, B)$ is stabilizable and $(A, C)$ is detectable. The following canonical factorizations hold: $$ I + H^*(z^{-*})H(z) = \Delta_1^*(z^{-*}) \Delta_1(z), \hspace{3mm} I + H(z)H^*(z^{-*}) = \Delta_2(z) \Delta_2^*(z^{-*}), $$ where we define the causal operators $$\Delta_1(z) = \Sigma_1^{1/2}(I + K_1^*(zI - A)^{-1}B),  \hspace{3mm} \Delta_2(z) = (I + C(zI - A)^{-1}K_2)\Sigma_2^{1/2}$$ and define
$$ K_1 = \Sigma_1^{-1} B^* P_1 A, \hspace{3mm} \Sigma_1 = I + B^*P_1B,$$
$$ K_2 = AP_2C^*\Sigma_2^{-1}, \hspace{3mm} \Sigma_2 = I + CP_2C^*,$$
where $P_1$ and $P_2$  are the solutions to the Riccati equations $$C^*C - P_1 + A^*P_1A - A^* P_1 B (I + B^*P_1B)^{-1}  B^* P_1 A = 0$$ and $$BB^* - P_2 + AP_2A^* - AP_2C^* (I + CP_2C^*)^{-1} A^*P_2C = 0.$$ 

\end{lemma}
\begin{proof}
We first factor $ I + H^*(z^{-*})H(z)$ as $ \Delta_1^*(z^{-*}) \Delta_1(z)$. We rewrite $ I + H^*(z^{-*})H(z)$ as $$  \begin{bmatrix} B^*(z^{-1}I - A^*)^{-1} & I \end{bmatrix} \begin{bmatrix} C^*C & 0 \\ 0 & I \end{bmatrix} \begin{bmatrix} (zI - A)^{-1}B \\ I \end{bmatrix}.$$ 
In light of Lemma \ref{equivalence-class-lemma-1}, we see that this equals
$$  \begin{bmatrix} B^*(z^{-1}I - A^*)^{-1} & I \end{bmatrix} \Lambda_1(P_1) \begin{bmatrix} (zI - A)^{-1}B \\ I \end{bmatrix}.$$
where $P_1$ is an arbitrary Hermitian matrix and we define $$\Lambda_1(P_1) = \begin{bmatrix} C^*C - P_1 + A^*P_1A & A^*P_1B \\ B^*P_1A & I + B^*P_1B \end{bmatrix}.$$ Notice that the $\Lambda_1(P_1)$ can be factored as $$\begin{bmatrix} I & K_1^*(P_1) \\ 0 & I \end{bmatrix} \begin{bmatrix} \Gamma_1(P_1) & 0 \\ 0 & \Sigma_1(P_1) \end{bmatrix} \begin{bmatrix} I & 0 \\ K_1(P_1) & I \end{bmatrix}, $$
where we define $$\Gamma_1(P_1) =  C^*C - P_1 + A^*P_1A - K_1^*(P_1)\Sigma_1(P_1) K_1(P_1),$$ 
$$ K_1(P_1) = \Sigma_1(P_1)^{-1} B^* P_1 A, \hspace{3mm} \Sigma_1(P_1) = I + B^*P_1B.$$
By assumption, $(A, B)$ is stabilizable and $(A, C)$ is detectable, therefore $(A^*, C^*)$ is stabilizable and $(A^*, B^*)$ is detectable. It follows that the Riccati equation $\Gamma_1(P_1) = 0$ has a unique stabilizing solution (see, e.g. Theorem E.6.2 in ``Linear Estimation" by Kailath, Sayed, and Hassibi).  Suppose $P_1$ is chosen to be this solution, and define $K_1 = K_1(P_1)$, $\Sigma_1 = \Sigma_1(P_1)$. We immediately obtain the canonical factorization $ I + H^*(z^{-*})H(z)  =  \Delta_1^*(z^{-*})\Delta_1(z), $ where we define 
\begin{equation} \label{delta-1-definition}
\Delta_1(z) = \Sigma_1^{1/2}(I + K_1(zI - A)^{-1}B).
\end{equation}

The factorization $ I + H(z)H^*(z^{-*}) =  \Delta_2(z)\Delta_2^*(z^{-*})$ can be obtained almost identically. We rewrite $ I + H(z) H^*(z^{-*})$ as
$$  \begin{bmatrix} C(zI - A)^{-1} & I \end{bmatrix} \begin{bmatrix} BB^* & 0 \\ 0 & I \end{bmatrix} \begin{bmatrix} (z^{-1}I - A^*)^{-1}C^* \\ I \end{bmatrix}.$$
In light of Lemma \ref{equivalence-class-lemma-1}, we see that this equals
$$ \begin{bmatrix} C(zI - A)^{-1} & I \end{bmatrix} \Lambda_2(P_2) \begin{bmatrix} (z^{-1}I - A^*)^{-1}C^* \\ I \end{bmatrix},$$
where $P_2$ is an arbitrary Hermitian matrix and we define $$\Lambda_2(P_2) = \begin{bmatrix} BB^* - P_2 + AP_2A^* & AP_2C^* \\ CP_2A^* & I + CP_2C^*\end{bmatrix}.$$ Notice that the $\Lambda_2(P_2)$ can be factored as $$\begin{bmatrix} I & K_2(P_2) \\ 0 & I \end{bmatrix} \begin{bmatrix} \Gamma_2(P_2) & 0 \\ 0 & \Sigma_2(P_2) \end{bmatrix} \begin{bmatrix} I & 0 \\ K_2^*(P_2) & I \end{bmatrix}, $$
where we define $$\Gamma_2(P_2) = BB^* - P_2 + AP_2A^* - K_2(P_2)\Sigma_2(P_2) K_2^*(P_2),$$ 
$$ K_2(P_2) = AP_2C^*\Sigma_2(P_2)^{-1}, \hspace{3mm} \Sigma_2(P_2) = I + CP_2C^*.$$
By assumption, $(A, B)$ is stabilizable and $(A, C)$ is detectable, therefore the Riccati equation $\Gamma_2(P_2) = 0$ has a unique stabilizing solution.  Suppose $P_2$ is chosen to be this solution, and define $K_2 = K_2(P_2)$, $\Sigma_2 = \Sigma_2(P_2)$. We immediately obtain the canonical factorization $ I + H(z) H^*(z^{-*}) = \Delta_2(z) \Delta_2^*(z^{-*}), $ where we define 
\begin{equation} \label{delta-2-definition}
\Delta_2(z) = (I + C(zI - A)^{-1}K_2 )\Sigma_2^{1/2}.
\end{equation}
\end{proof}


\begin{lemma} \label{center-factorization-lemma}
The following canonical factorization holds: $$\gamma^{-2}I_m + \gamma^{-4}J\Delta_1^{-1}\Delta_1^{-*}J^* = \Delta_3^*\Delta_3,$$ where $\Delta_1(z)$ is defined as in Lemma \ref{delta-factorization-lemma} and $\Delta_3(z)$ is the causal operator $$\Delta_3(z) = \gamma^{-1}\Sigma_3^{1/2}(I + K_3(zI - A_1)^{-1}A_1W_1L^*), $$ where we define $A_1 = A - BK_1$ and $W_1$ is  the solution of the Lyapunov equation $$  \gamma^{-2}B\Sigma^{-1}B^* - W_1 + A_1W_1A_1^*  = 0,$$ and we define $K_3 = K_3(P_3)$ and $\Sigma_3 = \Sigma_3(P_3)$
where $$ K_3(P_3) = \Sigma_3(P_3)^{-1} (L  + LW_1A_1^*P_3 A_1), \hspace{3mm} \Sigma_3(P_3) = \Sigma_0 + LW_1A_1^*P_3 A_1W_1L^*$$ and $P_3$ is the solution of the Riccati equation 
$$-P_3 + A_1^*P_3A_1 - K_3^*(P_3)\Sigma_3(P_3) K_3(P_3) = 0.$$

\end{lemma}

\begin{proof}
Recall that $$\Delta_1(z) = \Sigma_1^{1/2}(I + K_1(zI - A)^{-1}B),$$ where $K_1, \Sigma_1$ are defined in Lemma \ref{delta-factorization-lemma}. 
Therefore $$J(z)\Delta_1^{-1}(z) = L(zI - A_1)^{-1}B\Sigma_1^{-1/2}, $$ where we define $A_1 = A - BK_1$. We can hence rewrite $I + \gamma^{-2}J(z)\Delta_1^{-1}(z)\Delta_1^{-*}(z^{-*})J^*(z^{-*})$ as $$  \begin{bmatrix} L(zI - A_1)^{-1} & I \end{bmatrix} \begin{bmatrix} \gamma^{-2}B\Sigma_1^{-1}B^* & 0 \\ 0 & I \end{bmatrix} \begin{bmatrix} (z^{-1}I - A_1^*)^{-1}L^* \\ I \end{bmatrix}.$$ In light of Lemma \ref{equivalence-class-lemma-2}, the center matrix can be replaced by $$\begin{bmatrix} \gamma^{-2}B\Sigma_1^{-1}B^* - W_1 + A_1PA_1^* & A_1W_1L^*  \\  LW_1A_1^* & I + LW_1L^* \end{bmatrix},  $$ where $W_1$ is an arbitrary Hermitian matrix of appropriate dimensions. Suppose $W_1$ is chosen as the solution of the Lyapunov equation $$  \gamma^{-2}B\Sigma^{-1}B^* - W_1 + A_1W_1A_1^*  = 0.$$ Let $A_1W_1L^* = A_1W_1L^*$. We see that $I + \gamma^{-2}J(z)\Delta_1^{-1}(z)\Delta_1^{-*}(z^{-*})J^*(z^{-*})$ can be written as $$I + LW_1L^* + L(zI - A_1)^{-1}A_1W_1L^* + LW_1A_1^*(z^{-1} - A_1^*)^{-1}L^*.$$ This in turn can be expressed as $$ \begin{bmatrix} LW_1A_1^*(z^{-1}I - A_1^*)^{-1} & I \end{bmatrix} \begin{bmatrix} 0 & L^* \\ L & \Sigma_0 \end{bmatrix} \begin{bmatrix} (zI - A_1)^{-1}A_1W_1L^* \\ I \end{bmatrix}, $$ where we define $\Sigma_0 = I + LW_1L^*$.  Applying Lemma \ref{equivalence-class-lemma-1}, the center matrix can be replaced by $$\Lambda_3(P_3) = \begin{bmatrix} -P_3 + A_1^*P_3A_1 & L^* + A_1^*P_3A_1W_1L^*\\ L  + LW_1A_1^*P_3 A_1 & I + LW_1L^* + LW_1A_1^*P_3 A_1W_1L^* \end{bmatrix},$$ where $P_3$ is an arbitrary Hermitian matrix. Notice that $\Lambda_3(P_3)$ factors as $$\begin{bmatrix} I & K_3^*(P_3) \\ 0 & I \end{bmatrix} \begin{bmatrix} \Gamma_3(P_3) & 0 \\ 0 & \Sigma_3(P_3) \end{bmatrix} \begin{bmatrix} I & 0 \\ K_3(P_3) & I \end{bmatrix}, $$
where we define $$\Gamma_3(P_3) =  -P_3 + A_1^*P_3A_1 - K_3^*(P_3)\Sigma_3(P_3) K_3(P_3),$$ 
$$ K_3(P_3) = \Sigma_3(P_3)^{-1} (L  + LW_1A_1^*P_3 A_1), \hspace{3mm} \Sigma_3(P_3) = \Sigma_0 + LW_1A_1^*P_3 A_1W_1L^*.$$
Notice that $A_1$ is stable, therefore the Riccati equation $\Gamma_3(P_3) = 0$ has a unique stabilizing solution.  Suppose $P_3$ is chosen to be this solution, and define $K_3 = K_3(P_3)$, $\Sigma_3 = \Sigma_3(P_3)$. We immediately obtain the canonical factorization 
$ \gamma^{-2}I_m + \gamma^{-4}J(z)\Delta_1^{-1}(z)\Delta_1^{-*}(z^{-*})J^*(z^{-*}) = \Delta_3^*(z^{-*})\Delta_3(z), $ where we define 
\begin{equation*}
\Delta_3(z) = \gamma^{-1}\Sigma_3^{1/2}(I + K_3(zI - A_1)^{-1}A_1W_1L^*).
\end{equation*}

\end{proof}

\begin{lemma} \label{Q-decomposition-lemma}
Let $\Delta_3(z)$ be defined as in Lemma \ref{center-factorization-lemma}. Then $\Delta_3(z)Q(z)$ can be decomposed into causal and anticausal components as 
$$ \Delta_3(z)Q(z) = 
  LW_2C^*\Sigma_2^{-1/2} + L(zI - A)^{-1}AW_2
C^*\Sigma_2^{-1/2} 
 + LW_2A_2^*(z^{-1}I - A_2^*)^{-1} C^*\Sigma_2^{-1/2},
$$
where $W_2$ is the solution of the Lyapunov equation $$ \hat{B}B^* - W_2 +  \hat{A}W_2A_2^* = 0,$$ and we define $K_1, K_2$ as in Lemma \ref{delta-factorization-lemma}, $K_3, \Sigma_3$ as in Lemma \ref{center-factorization-lemma} and let $A_1 = A - BK_1, A_2 = A - K_1C$ and $$\hat{A} = \begin{bmatrix} A_1 & A_1W_1L^*L \\ 0 & A \end{bmatrix}, \hspace{3mm} \hat{B} = \begin{bmatrix} 0 \\ B \end{bmatrix}, \hspace{3mm} \hat{L} = \begin{bmatrix} \gamma^{-1}\Sigma_3^{1/2}K_3 & \gamma^{-1} \Sigma_3^{1/2}L \end{bmatrix}.$$
\end{lemma}
\begin{proof}
In Lemma \ref{delta-factorization-lemma} we found $\Delta_2(z)$: $$\Delta_2(z) = (I + C(zI - A)^{-1}K_2)\Sigma_2^{1/2}. $$ 
Therefore $$\Delta_2^{-*}(z^{-*}) = (I - K_2^*(z^{-1} - A_2^*)^{-1}C^*)\Sigma_2^{-1/2}. $$ It follows that $$H^*(z^{-*})\Delta_2^*(z^{-*}) = B^*(z^{-1}I - A_2^*)^{-1}C^*\Sigma_2^{-1/2}. $$ 
Similarly, $$ \Delta_3(z)J(z) = \hat{L}(zI - \hat{A})^{-1}\hat{B},$$ where we define $$\hat{A} = \begin{bmatrix} A_1 & A_1W_1L^*L \\ 0 & A \end{bmatrix}, \hspace{3mm} \hat{B} = \begin{bmatrix} 0 \\ B \end{bmatrix}, \hspace{3mm} \hat{L} = \begin{bmatrix} \gamma^{-1}\Sigma_3^{1/2}K_3 & \gamma^{-1} \Sigma_3^{1/2}L \end{bmatrix}.$$
Therefore, $\Delta_3(z)Q(z)$ can be rewritten as $$ \begin{bmatrix} \hat{L}(zI - \hat{A})^{-1} & I \end{bmatrix} \begin{bmatrix} \hat{B}B^* & 0 \\ 0 & 0 \end{bmatrix} \begin{bmatrix} (z^{-1}I - A_2^*)^{-1} C^*\Sigma_2^{-1/2} \\ I \end{bmatrix}.$$
In light of Lemma \ref{general-equivalence-class-lemma}, we see that the center matrix can be replaced by $$\begin{bmatrix} \hat{B}B^* - W_2 + \hat{A}W_2A_2^* & \hat{A}W_2C^*\Sigma_2^{-1/2}  \\  \hat{L}W_2A_2^* & \hat{L}W_2C^*\Sigma_2^{-1/2} \end{bmatrix},  $$ where $W_2$ is an arbitrary matrix of appropriate dimension. We take $W_2$ to be the solution of the Sylvester equation $$ \hat{B}B^* - W_2 +  \hat{A}W_2A_2^* = 0,$$ and immediately obtain the desired decomposition of $\Delta_3(z)Q(z)$ into causal and anticausal components.
\end{proof}

\begin{lemma} \label{equivalence-class-lemma-1}
For all $H, F$ and all Hermitian matrices $P$, we have
$$\begin{bmatrix} H^*(z^{-1}I - F^*)^{-1} & I \end{bmatrix} \Omega(P) \begin{bmatrix} (zI - F)^{-1}H \\ I \end{bmatrix} = 0, $$ where we define $$\Omega(P) = \begin{bmatrix} -P + F^*PF & F^*PH \\ H^*PF & H^*PH \end{bmatrix}. $$
\end{lemma}

\begin{proof}
This identity is essentially the ``transpose" of Lemma \ref{equivalence-class-lemma-2} and is easily verified via direct calculation.
\end{proof}

\begin{lemma}\label{equivalence-class-lemma-2}
For all $H, F$ and all Hermitian matrices $P$, we have
$$\begin{bmatrix} H(zI - F)^{-1} & I \end{bmatrix} \Omega(P) \begin{bmatrix} (z^{-1}I - F^*)^{-1}H^* \\ I \end{bmatrix} = 0, $$ where we define $$\Omega(P) = \begin{bmatrix} -P + FPF^* & FPH^* \\ HPF^* & HPH^* \end{bmatrix}. $$
\end{lemma}

\begin{proof}
This identity is a special case of Lemma \ref{general-equivalence-class-lemma}; it also appears as Lemma 12.3.3 in ``Indefinite-Quadratic Estimation and Control" by Hassibi, Sayed, and Kailath.
\end{proof}

\begin{lemma} \label{general-equivalence-class-lemma}
For all $H_1, H_2, F_1, F_2$ and all matrices $W$, we have
$$\begin{bmatrix} H_1(zI - F_1)^{-1} & I \end{bmatrix} \Omega(W) \begin{bmatrix} (z^{-1}I - F_2^*)^{-1}H_2^* \\ I \end{bmatrix} = 0, $$ where we define $$\Omega(W) = \begin{bmatrix} -W + F_1WF_2^* & F_1WH_2^* \\ H_1WF_2^* & H_1WH_2^* \end{bmatrix}. $$
\end{lemma}

\begin{proof}
Notice that $\Omega(W)$ can be rewritten as $$\Omega(W) = \begin{bmatrix} F_1 \\ H_1 \end{bmatrix} W \begin{bmatrix} F_2 & H_2 \end{bmatrix} - \begin{bmatrix} I \\0 \end{bmatrix} W \begin{bmatrix} I & 0 \end{bmatrix}.$$
The proof is immediate after observing that $$\begin{bmatrix} H_1(zI - F_1)^{-1} & I \end{bmatrix}\begin{bmatrix} F_1 \\ H_1 \end{bmatrix} = H_1(zI - F_1)^{-1}z, \hspace{3mm} \begin{bmatrix} F_2 & H_2 \end{bmatrix} \begin{bmatrix} (z^{-1}I - F_2^*)^{-1}H_2^* \\ I \end{bmatrix} = z^{-1}(z^{-1}I - F_2^*)^{-1}H_2^*.$$
\end{proof}

\end{document}